\documentclass[runningheads]{llncs}

\newcommand {\ent} {\mathrel{{\scriptstyle\mid\!\sim}}}

\newcommand {\sx} {\langle}
\newcommand {\dx} {\rangle}

\newcommand {\emme} {\mathcal{M}}

\newcommand {\vuoto} {\emptyset}

\newcommand{\tip}{{\bf T}}

\newcommand{\lc}{\mathcal{LC}}
\newcommand{\alc}{\mathcal{ALC}}

\newcommand{\alctmin}{\mathcal{ALC}+\tip_{min}}
\newcommand{\el}{\mathcal{EL}^{\bot}}
\newcommand{\elpb}{{\mathcal{EL}}^{+}_{\bot}}

\newcommand{\alctr}{\mathcal{ALC}+\tip_{\textsf{\tiny R}}}

\newcommand{\be}{\begin{enumerate}}
\newcommand{\ee}{\end{enumerate}}

\newcommand{\hide}[1]{}

\def \cases{\left \{\begin{array}{l}}
\def \endcases{\end{array}\right .}

\newcommand {\ri} {\rightarrow}

\newcommand {\bes} {\begin{description}}
\newcommand{\ens} {\end{description}}

\newcommand {\beq} {\begin{quote}}
\newcommand {\enq} {\end{quote}}
\newcommand {\bit} {\begin{itemize}}
\newcommand {\enit} {\end{itemize}}

\usepackage{times}
\usepackage{helvet}
\usepackage{courier}
\usepackage{graphicx}
\usepackage{latexsym}
\usepackage{graphicx}
\usepackage{color}
\usepackage{amssymb}
\usepackage{colortbl}
\usepackage{amsmath}
\usepackage{microtype}%if unwanted, comment out or use option "draft"
\usepackage{amsmath}

\def \ri{\rightarrow}

\begin{document}
\bibliographystyle{plain}

\title{A conditional, a fuzzy and a probabilistic interpretation \\
of self-organising maps}

\author{Laura Giordano \inst{1} \and Valentina Gliozzi \inst{2} \and Daniele Theseider Dupr{\'{e}}  \inst{1}}

\institute{DISIT - Universit\`a del Piemonte Orientale, 
 Alessandria, Italy \\\email{laura.giordano@uniupo.it}, \ \  \email{dtd@di.unipmn.it}
 \and
Center for Logic, Language and Cognition, 
Dipartimento di Informatica, \\
Universit\`a di Torino, Italy, \\ \email{laura.giordano@uniupo.it}
}

\authorrunning{ }
\titlerunning{ }

 \maketitle

\begin{abstract} 
 In this paper  we establish a link between fuzzy and preferential semantics
 for description logics and  Self-Organising Maps, which have been proposed as possible candidates to explain the psychological mechanisms underlying category generalisation.
In particular, we show that the input/output behavior of a Self-Organising Map after training can be described by a fuzzy description logic interpretation 
as well as by a preferential interpretation, based on a concept-wise multipreference semantics, which takes into account preferences with respect to different concepts and has been recently proposed for ranked and for weighted defeasible description logics. 
Properties of the network can be proven by model checking on the fuzzy or on the preferential interpretation. 
 Starting from the fuzzy interpretation, we also provide a probabilistic account for this neural network model.

\end{abstract}

\newpage

\vspace{-0.1cm}
\section{Introduction}
\vspace{-0.1cm}

Conditional logics  
have their roots in philosophical logic.
They have  been studied
first by Lewis \cite{Lewis:73,Nute80}  to formalize 
hypothetical and counterfactual reasoning (if $A$ were the case then $B$) that
cannot be captured by classical  logic with its material
implication. From the 80's they have been considered in
computer science and artificial intelligence 
and they have provided an axiomatic foundation of non-monotonic  and 
common sense reasoning \cite{Delgrande:87,Makinson88,Pearl:88,KrausLehmannMagidor:90,Pearl90,whatdoes,BenferhatIJCAI93}.

KLM preferential approaches  \cite{KrausLehmannMagidor:90,whatdoes,Lehmann95} to common sense reasoning
have been more recently extended to description logics, to deal with inheritance with exceptions in ontologies,
allowing for non-strict forms of inclusions,
called {\em typicality or defeasible inclusions}, corresponding to conditional implications, based on different preferential semantics \cite{lpar2007,sudafricaniKR,FI09} 
and closure constructions \cite{casinistraccia2010,CasiniDL2013,dl2013,AIJ15,Pensel18,CasiniStracciaM19,BritzCMMSV21,AIJ21}. %CasiniDL2013,dl2013,

Fuzzy description logics have been widely studied in the literature for representing vagueness in DLs \cite{Straccia05,Stoilos05,LukasiewiczStraccia09,PenalosaARTINT15,BobilloStracciaEL18}, based on the idea that concepts and roles can be interpreted 
as fuzzy sets and fuzzy binary relations.

In this paper we aim at developing a logical interpretation of self-organising maps (SOMs) \cite{kohonen2001}, 
which have been proposed as possible candidates to explain the psychological mechanisms underlying category generalisation.
They are psychologically and biologically plausible neural network models that can also learn after limited exposure to positive category examples, without any need of contrastive information.
We consider a ``concept-wise" multi-preferential semantics  
 \cite{TPLP2020}, which has been recently introduced for a lightweight description logic of the $\el$ family \cite{rifel} and takes into account preferences with respect to different concepts. 
We show that both the fuzzy semantics and the multi-preferential semantics can be used to provide a logical interpretation of SOMs,
and allow for the verification of properties of the trained SOM by model checking. 

Both interpretations are based on the idea of associating each learned category to a concept in the language of 
the simple description logic $\lc$, which does not allow for roles and role restrictions, 
but allows for the boolean combination of concepts 
(as we will see, restricting to a language without roles is enough for the purpose of developing a logical interpretation of SOMs).
We show that the learning process in self-organising  
maps produces, as a result, either a {\em fuzzy model}, in which each  concept (or learned category) is interpreted as a fuzzy set over the domain of input stimuli,
or a {\em multipreference model}
by associating a preference relation to each concept (each learned category).
Both models can be exploited to  extract or validate 
knowledge from the empirical data used in the learning process 
and the evaluation of  such knowledge can be done by model checking, using the information recorded in the SOM.
The verification of logical properties of a neural network
can be useful 
for post-hoc explanation,
in view of a trustworthy, reliable and  explainable AI \cite{Adadi18,Guidotti2019,Arrieta2020}.

Concerning the preferential semantics, based on the assumption that the abstraction process in the SOM is able to identify the most typical exemplars for a given category,
in the semantic representation of a category, we identify some specific stimuli 
as the {\em typical exemplars} of the category, and define a preference relation among the exemplars of a category.  
To this purpose, 
we use the notion of distance of an input stimulus from a category representation.
We then exploit a notion of relative distance, introduced by Gliozzi and Plunkett in their similarity-based account of category generalization based on self-organising maps \cite{CogSci2017}, for developing another semantic interpretation of SOMs based on {\em fuzzy DL interpretations}.
\normalcolor
This is done by interpreting each category (concept) as a function mapping each input stimulus to a value in $[0,1]$, based on the map's generalization degree of category membership to the stimulus used in \cite{CogSci2017}. Our fuzzy interpretation of SOMs sticks to this specific use for category generalization.

The multipreference model of the SOM will be defined as a multipreference $\lc$ interpretation, while 
the fuzzy model of the SOM will be defined as a fuzzy $\lc$ interpretation. 
In both cases, model checking can be used for the verification of inclusions (either defeasible inclusions or fuzzy inclusion axioms) over the respective models of the SOM.
Starting from the fuzzy interpretation of the SOM we also provide a probabilistic interpretation of this neural network model  based on  Zadeh's  probability of fuzzy events  \cite{Zadeh1968}. The paper extends the 
results in \cite{CILC2020} which only discusses a preferential interpretation of SOMs.

The paper is organized as follows. Section \ref{som-general} shortly describes self-organising maps.  Section \ref{sec:LC} and \ref{sezione:fuzzyDL} contain preliminaries about the description logic ${\lc}$ and fuzzy ${\lc}$ interpretations. A concept-wise multipreference semantics for ${\lc}$ with typicality inclusions is described in Section  \ref{sec:multipref}. Section \ref{sec:preference-modelSOM} and Section \ref{sec:fuzzy_SOM},  respectively, relate self-organising maps with multipreference and a fuzzy  DL interpretations, while Section  \ref{sec:probab} provides a probabilistic interpretation of SOMs.
Section  \ref{sec:revision} hints at a possible relation between the process of updating the category representation in the SOM with change operators in knowledge representation. Section  \ref{conclusions}  concludes the paper and discusses related work.

\normalcolor

\section{Self-organising maps}\label{som-general}

Self-organising maps (SOMs, introduced by Kohonen \cite{kohonen2001}) are particularly plausible neural network models that learn in a human-like manner. In particular: SOMs learn to organize {stimuli into} categories in an {\em unsupervised} way, without the need of a teacher providing a feedback;
they can learn with just {a few} positive stimuli, without the need for negative examples or contrastive information; they reflect basic constraints of a plausible brain implementation in different areas of the cortex \cite{miikkulainen2005}, and are therefore biologically plausible models of category formation; they have proven to be capable of explaining experimental results. 

In this section we shortly describe the architecture of SOMs and report Gliozzi and Plunkett's  similarity-based account of category generalization based on SOMs \cite{CogSci2017}. In brief, in  \cite{CogSci2017} the authors judge a new stimulus as belonging  to a category by comparing the distance of the stimulus from the category representation to the precision of the category representation. 

SOMs  consist of a set of neurons, or units, spatially organized in a grid \cite{kohonen2001}, as in Figure $1$.
\begin{figure}[htbp]
\centering
\includegraphics[width=0.6\textwidth]{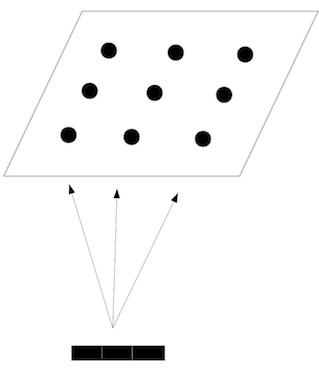}
\caption{An example of SOM. The set of rectangles stands for the input presented to the SOM (in the example the input is three-dimensional). This is presented to {\em all} neurons of the SOM (these are the neurons in the upper grid) in order to find the $BMU$.}
\end{figure}
Each map unit $u$ is associated with a
weight vector $w_u$ of the same dimensionality as the input vectors.  
At the beginning of training, all weight vectors are initialized to random values, outside the range of values of the input stimuli. 
During training, the input elements (a set $X$) are sequentially presented to all neurons of the map. After each presentation of an input $x$, the {\em best-matching unit} (BMU$_x$) is selected: this is a unit $i$ whose weight vector $w_i$ is closest to the stimulus $x$ (i.e. it minimizes  the distance $\|x - w_j\|$). 

The weights of the best matching unit and of its surrounding units are updated in order to maximize the chances that the same unit (or its surrounding units) will be selected as the best matching unit for the same stimulus or for similar stimuli {on subsequent presentations}. In particular,  the weight update reduces the distance between the best matching unit's weights (and its surrounding neurons' weights) and the incoming input. Furthermore, it organizes the map topologically so that  weights of close-by neurons are updated in a similar direction, and come to react to similar inputs.

The learning process is incremental: after the presentation of each input, the map's representation of the input (and in particular the representation of its best-matching unit) is updated in order to take into account the new incoming stimulus.
At the end of the whole process, the SOM has learned to organize the stimuli in a topologically significant way: similar inputs (with respect to Euclidean distance) are mapped to close by areas in the map, whereas inputs which are far apart from each other are mapped to distant areas of the map. 
%In the example of a bidimensional SOM categorizing animal descriptions in Figure $2$, one can notice that similar animals (e.g. dog and wolf) 
%are mapped onto close by areas of the map, whereas different animals (e.g. dog and goose) are mapped onto distant areas of the map, i.e. the map has learned to organize animal categories in a topologically meaningful way. We refer to \cite{kohonen2001} for a detailed description of SOMs.
%%
%\begin{figure}[htbp]
%\centering
%\includegraphics[width=0.8\textwidth]{SOManimali2.png}
%\caption{Example of a bidimensional SOM categorizing animal descriptions.} 
%\end{figure}

Once the SOM has learned to categorize, in order to assess category generalization, Gliozzi and Plunkett \cite{CogSci2017} define the map's disposition to consider a new stimulus $y$ as a member of a known category $C$. This disposition is defined as a function of the {\em distance} of $y$ from the {\em map's representation} of $C$.
They take a minimalist notion of what is the map's category representation: this is the ensemble of best-matching units corresponding to known instances of the category. 
They use $BMU_{C}= \{ BMU_x \mid x \in X \mbox{ and } x \in C_i\}$ to refer to the map's representation of category $C$ and define category generalization as depending on two elements:
\begin{itemize} 
\item the distance of the new stimulus $y$ ($y \not \in X$) with respect to the category representation: $min \|y - BMU_{C}\|$ (in the following, also denoted by $d(y,C_i)$)
\item {\em compared to} the maximal distance from that representation of all known instances of the category
\end{itemize} 
This captured by the following notion of {\em relative distance} ({\em rd} for short)  \cite{CogSci2017} :

\begin{equation}
\label{relative-distance}
rd(y,C) = \frac{min \|y - BMU_{C}\| }{max_{x \in C} \| x - BMU_{x}\| }
\end{equation}
where $min \|y - BMU_{C}\|$ is the (minimal) Euclidean distance between $y$ and $C$'s category representation, and ${max_{x \in C} \| x - BMU_{x}\| }$ expresses the {\em precision} of category representation,
and is the (maximal) Euclidean distance between any known member of the category and the category representation. 

With this definition, a given Euclidean distance from $y$ to $C's$ category representation will give rise to a higher {\em relative distance rd} if category representation is precise (and the maximal distance between C and its known examples is low) than if category representation is coarse (and the maximal distance between C and its known examples is high). 
As a function of the relative distance above, Gliozzi and Plunkett  then define the  {\em map's Generalization Degree} of category $C$ membership to a new stimulus $y$. 
This is a function of the relative distance of Equation (\ref{relative-distance}).  The map's Generalization Degree exponentially decreases with the increase of the relative distance as follows:

\begin{equation}
\label{Generalization Degree}
\mbox{Generalization Degree= }e^{-rd(y,C)} 
\end{equation}

It was observed that the above notion of relative distance (Equation \ref{relative-distance}) requires there to be a memory of some of the known instances of the category being used (this is needed to calculate the denominator in the equation). This gives rise to a sort of hybrid model in which category representation and some exemplars coexist. 
An alternative way of formulating the same notion of relative distance  would be to calculate {\em online} the distance between known category instance currently examined and the representation of the category being formed. 

By judging a new stimulus as belonging  to a category by comparing the distance of the stimulus from the category representation to the precision of the category representation, Gliozzi and Plunkett demonstrate  \cite{CogSci2017}  that the Numerosity and Variability effects of category generalization, described by Griffiths and Tenenbaum \cite{tengrif2001}, and usually explained with Bayesian tools, can be {accommodated} within a simple and psychologically plausible similarity-based account. 
In the next sections, we show that the notions of distance and relative distance can be used as a basis for a logical semantics for SOMs.

\section{The description logic $\lc$}\label{sec:LC}

In this section 
and in Section  \ref{sezione:fuzzyDL} we will consider the boolean fragment of the description logic $\alc$ \cite{handbook}  as well as its fuzzy extension \cite{LukasiewiczStraccia09}.
As we will see, the boolean fragment does not allow for roles but it is expressive enough for the purpose of providing a logical interpretation of a SOM, in which categories are interpreted as atomic concepts and their combination through union, intersection and complement allows for the formulation of properties to be validated over a SOM model. 

Description logics are based on a simple set-theoretic semantics and, in the following, we will only focus on the semantics of boolean concepts.  In the two-valued case, concepts are interpreted as sets over a domain $\Delta$ of elements, that is, each concept (e.g., $\mathit{Elephant}$) is interpreted as a set of elements (the set of all elephants in $\Delta$).  In the fuzzy case, concepts are interpreted as fuzzy sets. For the two-valued case, the interpretation of union, intersection and complement is the usual one in set theory, while in the fuzzy case it depends on the underlying fuzzy logic combination functions. Individual names ($\mathit{dumbo}$, $\mathit{garfield}$, etc.)  represent specific domain elements in $\Delta$.

Let $ \lc$ be the fragment of the description logic $\alc$ which does not admit roles and universal and existential restrictions, namely, the fragment only containing union, intersection, and complement as concept constructors. 
While, as mentioned above, we restrict our consideration to the fragment of $\alc$ without roles, 
both the fuzzy and the preferential description logics considered in the following have been first introduced for description logics including roles. 

Let ${N_C}$ be a set of concept names 
  and ${N_I}$ a set of individual names.  
The set  of $\lc$ \emph{concepts} (or, simply, concepts) can be
defined inductively as follows:

 \begin{itemize}
\item
$A \in N_C$, $\top$ and $\bot$ are {concepts};
    
\item
if $C$ and $ D$ are concepts, 
then $C
\sqcap D, C \sqcup D, \neg C$ 
are {concepts}.
\end{itemize}

\noindent  
A knowledge base (KB) $K$ is a pair $({\cal T}, {\cal A})$, where ${\cal T}$ is a TBox -- i.e., a set of {\em concept inclusions} (or subsumptions) $C \sqsubseteq D$, where $C,D$ are concepts --
and
${\cal A}$ is an ABox -- i.e., a set of {\em assertions} of the form $C(a)$ 
where $C$ is a  concept and $a$ an individual name in $N_I$.

As an example, given the concept names $\mathit{Elephant}$, $\mathit{White\_Animal}$, $\mathit{African\_}$ $\mathit{ Animal}$ and $\mathit{ Young}$, the complex concepts $\mathit{Elephant \sqcap White\_Animal}$ and $\mathit{Elephant \sqcap }$ $\mathit{(White\_}$ $\mathit{Animal \sqcup African\_Animal)}$ represent, respectively, the set of white elephants and the set of elephants which are white or live in Africa.
For an individual name, $\mathit{dumbo}$, the assertion $\mathit{(Elephant \sqcap White\_}$ $\mathit{Animal)(dumbo)}$ means that Dumbo is a white elephant. The concept inclusion $\mathit{Young \sqcap White\_Animal \sqsubseteq African\_Animal}$ means that all young white elephants live in Africa (the set of young white elephants is a subset of the set of African animals).

An  $\lc$ {\em interpretation}  is defined as for $\alc$ as a pair $I=\langle \Delta, \cdot^I \rangle$ where:
$\Delta$ is a domain---a set whose elements are denoted by $x, y, z, \dots$---and 
$\cdot^I$ is an extension function that maps each
concept name $C\in N_C$ to a set $C^I \subseteq  \Delta$, 
and each individual name $a\in N_I$ to an element $a^I \in  \Delta$.
It is extended to complex concepts  as follows:
\begin{align*}
&\top^I  =\Delta     \;\;\;\;\;\;     \bot^I =\vuoto \\
	 & (\neg C)^I=\Delta \backslash C^I\\
	 &(C \sqcap D)^I  =C^I \cap D^I  \\
	 & (C \sqcup D)^I =C^I \cup D^I
	%&(\forall R.C)^I =\{x \in \Delta \tc \forall y. (x,y) \in R^I \imp y \in C^I\} \\ %&(\neg C)^I =\Delta \backslash C^I \\
	%&(\exists R.C)^I =\{x \in \Delta \tc \exists y.(x,y) \in R^I \ \& \ y \in C^I\}.
\end{align*}
The notion of satisfiability of a KB  in an interpretation and the notion of entailment are defined as follows:

\begin{definition}[Satisfiability and entailment] \label{satisfiability}
Given an $\lc$ interpretation $I=\langle \Delta, \cdot^I \rangle$: 

\begin{quote}
	- $I$  satisfies an inclusion $C \sqsubseteq D$ if   $C^I \subseteq D^I$;
	
	-   $I$ satisfies an assertion $C(a)$ if $a^I \in C^I$;

\end{quote}

\noindent
 Given  a KB $K=({\cal T}, {\cal A})$, 
 an interpretation $I$  satisfies ${\cal T}$ (resp., ${\cal A}$) if $I$ satisfies all  inclusions in ${\cal T}$ (resp., all assertions in ${\cal A}$).
 $I$ is an $\lc$ \emph{model} of $K$ if $I$ satisfies ${\cal T}$ and ${\cal A}$.

 Letting a {\em query} $F$ to be either an inclusion $C \sqsubseteq D$ (where $C$ and $D$ are concepts) 
or an assertion $C(a)$, 
 {\em $F$ is entailed by $K$}, written $K \models_{\lc} F$, if for all $\lc$ models $I=$$\sx \Delta,  \cdot^I\dx$ of $K$,
$I$ satisfies $F$.
\end{definition}
Given a knowledge base $K$,
the {\em subsumption} problem is the problem of deciding whether an inclusion $C \sqsubseteq D$ is entailed by  $K$.
The {\em instance checking} problem is the problem of deciding whether an assertion $C(a)$ is entailed by $K$.

\section{Fuzzy $\lc$ interpretations}\label{sezione:fuzzyDL}

Fuzzy description logics allow to represent vagueness in DLs \cite{Straccia05,Stoilos05,LukasiewiczStraccia09,PenalosaARTINT15,BobilloStracciaEL18}, by interpreting concepts and roles  
as fuzzy sets.
As in Mathematical Fuzzy Logic \cite{Cintula2011} a formula has a degree of truth in an interpretation, rather than being either true or false,
in a fuzzy DL, axioms 
are associated to a degree of truth (typically in the interval  $[0, 1]$). 
In the following we shortly recall the semantics of a fuzzy extension of $\lc$, 
as a fragment of fuzzy $\alc$, referring to the survey by Lukasiewicz and Straccia \cite{LukasiewiczStraccia09}.
We limit our consideration to $\lc$ constructs and  only to the few features of a fuzzy DLs
which are relevant for defining an interpretation of SOMs and, in particular, we omit considering datatypes.

A {\em fuzzy interpretation} for $\lc$ is a pair $I=\langle \Delta, \cdot^I \rangle$ where:
$\Delta$ is a non-empty domain and 
$\cdot^I$ is {\em fuzzy interpretation function} that assigns to each
concept name $A\in N_C$ a function  $A^I :  \Delta \ri [0,1]$,
and to each individual name $a\in N_I$ an element $a^I \in  \Delta$.
A domain element $x \in \Delta$ 
belongs to the extension of $A$ to some degree in $[0, 1]$, i.e., $A^I$ is a fuzzy set.

The  interpretation function $\cdot^I$ is extended to complex concepts as follows: 

$\mbox{\ \ \ }$ $\top^I(x)=1$, $\bot^I(x)=0$,  

$\mbox{\ \ \ }$  $(\neg C)^I(x)=\Delta \ominus C^I(x)$, 

$\mbox{\ \ \ }$  $(C \sqcap D)^I(x) =C^I(x) \otimes D^I(x)$  

$\mbox{\ \ \ }$  $(C \sqcup D)^I(x) =C^I(x) \oplus D^I(x)$   

\noindent
(for $x \in \Delta$), and to non-fuzzy axioms (i.e., to strict inclusions and assertions of an $\lc$ knowledge base) as follows:

$\mbox{\ \ \ }$  $(C \sqsubseteq D)^I= inf_{x \in \Delta}  \; \; C^I(x) \rhd D^I(x)$

$\mbox{\ \ \ }$  $(C(a))^I=C^I(a^I)$

\noindent
where $\otimes$, $\oplus$, $\rhd$ and $\ominus$ are so-called ``combination functions, namely, triangular norms (or t-norms), triangular co-norms (or s-norms), implication functions, and negation functions, respectively, which extend the classical Boolean conjunction, disjunction, implication, and negation, respectively, to the many-valued case"  \cite{LukasiewiczStraccia09}.
They have been studied for multi-valued logics  \cite{Hahnle99} and for fuzzy DLs.  
 For instance, in both Zadeh and G\"odel logics $a \otimes b= min\{a,b\}$,  $a \oplus b= max\{a,b\}$. In Zadeh logic $a \rhd b= max\{1-a,b\}$ and $ \ominus a = 1-a$. In G\"odel logic  $a \rhd b= 1$ {\em if} $a \leq b$ {\em and} $b$ {\em otherwise};
 $ \ominus a = 1$ {\em if} $a=0$  {\em and} $0$ {\em otherwise}.
Following \cite{LukasiewiczStraccia09}, we will not commit to a specific fuzzy logic.

A {\em fuzzy $\lc$ knowledge base} $K$ is a pair $({\cal T}, {\cal A})$ where ${\cal T}$ is a fuzzy TBox  and ${\cal A}$ a fuzzy ABox. A fuzzy TBox is a set of {\em fuzzy concept inclusions} of the form $C \sqsubseteq D ~ \theta ~ n$, where $C \sqsubseteq D$ is an $\lc$ concept inclusion axiom, $\theta \in \{\geq,\leq,>,<\}$ and $n \in [0,1]$. A fuzzy ABox ${\cal A}$ is a set of {\em fuzzy assertions} of the form $C(a) ~\theta ~ n$, 
where $C$ is an $\lc$ concept, 
$a \in N_I$,  $\theta \in \{{\geq,}\leq,>,<\}$ and $n \in [0,1]$.
Following Bobillo and Straccia  \cite{BobilloStracciaEL18}, we assume that fuzzy interpretations are {\em witnessed}, i.e., the $\mathit{inf}$ (as well as the $\mathit{sup}$) is attained at some point of the involved domain.

In a fuzzy DL interpretation, an assertion like $\mathit{ White\_Animal(dumbo)}$ is interpreted as having a degree  in $[0, 1]$, rather than having a value $\mathit{true}$ or $\mathit{false}$. For instance, in a given interpretation $I$, one might have that dumbo 
is white with degree 0.6, i.e.,
$\mathit{White\_Animal^I(dumbo^I)}=0.6$, and that Dumbo is an elephant with degree $0.9$, $\mathit{Elephant^I}$ $\mathit{(dumbo^I)}=0.9$. In Zadeh logic, one would get 
$\mathit{((Elephant \sqcap}$ $\mathit{ White\_Animal)(dumbo))^I}=\mathit{(Elephant \sqcap White\_Animal)^I(dumbo^I)}= \mathit{min \{Ele}$- $\mathit{phant^I(dumbo^I), White\_Animal^I(dumbo^I)\} }=min(0.9,0.6)=0.6$, i.e., Dumbo belongs to the set of white elephants with degree $0.6$.
Hence,  a fuzzy assertion $\mathit{(Elephant}$ $\mathit{ \sqcap White\_Animal)(dumbo)} \geq 0.5$ would be satisfied in the interpretation $I$, according to the definition of satisfiability given below.

The notions of satisfiability of a KB  in a fuzzy interpretation and of entailment are defined in the natural way.
\begin{definition}[Satisfiability and entailment for fuzzy knowledge bases] \label{satisfiability}
A  fuzzy interpretation $I$ satisfies a fuzzy $\lc$ axiom $E$ (denoted $I \models E$), as follows.\\
 For $\theta \in \{\geq,\leq,>,<\}$:

- $I$ satisfies a fuzzy $\lc$ inclusion axiom $C \sqsubseteq D \;\theta\; n$ if $(C \sqsubseteq D)^I \theta\; n$;

- $I$ satisfies a fuzzy $\lc$ assertion $C(a) \; \theta \; n$ if $C^I(a^I) \theta\; n$.
 
\noindent
Given  a fuzzy KB $K=({\cal T}, {\cal A})$,
 a fuzzy interpretation $I$  satisfies ${\cal T}$ (resp. ${\cal A}$) if $I$ satisfies all fuzzy  inclusions in ${\cal T}$ (resp. all fuzzy assertions in ${\cal A}$).
A fuzzy interpretation $I$ is a \emph{model} of $K$ if $I$ satisfies ${\cal T}$ and ${\cal A}$.
A fuzzy axiom $E$   {is entailed by a fuzzy knowledge base $K$}, written $K \models E$, if for all models $I=$$\sx \Delta,  \cdot^I\dx$ of $K$,
$I$ satisfies $E$.
\end{definition}

\section{A concept-wise multipreference semantics for $\lc$ plus typicality} \label{sec:multipref}

In this section  we describe an extension of $\lc$ with typicality inclusions, defined along the lines of the extension of description logics with typicality \cite{lpar2007,AIJ15}, but under a different semantics with multiple preferences \cite{TPLP2020}.
This concept-wise multi-preference semantics has been originally introduced for the description logic ${\cal EL}^{+}_\bot$ of the ${\cal EL}$ family  \cite{rifel}, which is at the basis of OWL2 EL Profile. The semantics is a variant of the multipreferential semantics developed by Gliozzi  \cite{GliozziAIIA2016} to define a refinement of the Rational Closure semantics  for description logics.  In the following we reformulate concept-wise multipreference semantics for $\lc$. 

In addition to standard concept inclusions $C \sqsubseteq D$ (called  {\em strict} inclusions in the following), 
typicality inclusions are allowed having the form $\tip(C) \sqsubseteq D$, where $C$ and $D$ are $\lc$ concepts. 
A typicality inclusion $\tip(C) \sqsubseteq D$ means that ``typical C's are D's" or ``normally C's are D's" and corresponds to a conditional implication $C \ent D$ in Kraus, Lehmann and Magidor's (KLM) preferential approach \cite{KrausLehmannMagidor:90,whatdoes}. 
Such inclusions are defeasible, i.e.,  admit exceptions, while 
strict inclusions must be satisfied by all domain elements.

Let ${\cal C}= \{C_1, \ldots, C_k\}$ be a set of (general) $\lc$ concepts, called {\em distinguished concepts}. 
For each concept $C_i \in {\cal C}$, we introduce a modular preference relation $<_{C_i}$ which describes the preference among domain elements with respect to $C_i$.
Each preference relation $<_{C_i}$ has the same properties of preference relations in KLM-style ranked interpretations \cite{whatdoes}, it is a modular and well-founded strict partial order (i.e., an irreflexive and transitive relation). In particular, a preference relation  $<_{C_i}$ is {\em well-founded} 
if, for all $S \subseteq \Delta$, if $S\neq \emptyset$, then $min_{<_{C_i}}(S)\neq \emptyset$;
 relation  $<_{C_i}$ is {\em modular} if,
for all $x,y,z \in \Delta$, if $x <_{C_i} y$ then $x <_{C_i} z$ or $z <_{C_i} y$.

Observe that, as usual, the strict preference relation $<_{C_i}$ can be defined from a total preorder $\leq_{C_i}$ by letting 
$x <_{C_i} y$ iff $x \leq_{C_i} y$ and not $y \leq_{C_i} x$. An equivalence relation $\sim_{C_i}$ can be defined as: $x \sim_{C_i} y$ iff $x \leq_{C_i} y$ and  $y \leq_{C_i} x$.

  \begin{definition}[Multipreference interpretation]\label{interpretazione_Ci}  
A {\em multipreference interpretation}  is a tuple
$\emme= \langle \Delta, <_{C_1}, \ldots, <_{C_k}, \cdot^I \rangle$, 
where:
\begin{itemize}
\item[(a)] $\Delta$ is a non-empty domain;
 
\item[(b)] $<_{C_i}$ is an irreflexive, transitive, well-founded and modular relation over $\Delta$;

\item[(d)]  
$\cdot^I$ is an interpretation function, defined as in $\lc$ interpretations 
(see Section \ref{sec:LC}).
\end{itemize}

\end{definition}
Observe that, given a multipreference interpretation, a triple $\emme_{C_i}= \langle \Delta, <_{C_i}, \cdot^I \rangle$, which can be associated to each concept $C_i$, is a ranked interpretation as the ones considered for $\alc$ plus typicality in \cite{AIJ15}. 
The preference relation $<_{C_i}$ allows the set of {\em prototypical  $C_i$-elements} to be defined as the set of $C_i$-elements which are minimal with respect to $<_{C_i}$, i.e., the set $min_{<_{C_i}} (C_i^I)$.
As a consequence, the multipreference interpretation above is able to single out the typical $C_i$-elements, for all distinguished concepts $C_i \in {\cal C}$.

The multipreference  interpretations have been first introduced in \cite{TPLP2020}, to develop a semantics for ranked knowledge bases in a lightweight description logic 
which is at the basis of OWL2 EL Profile \cite{OWL}, based on an approach
inspired by Brewka's framework of basic preference descriptions  \cite{Brewka04}.
In the following we will shortly reformulate for $\lc$ the notion of  ranked knowledge base, 
and we will just recall the main ideas and results concerning concept-wise multipreference entailment 
without providing a reformulation for $\lc$.
In fact, in the following sections, we will not focus on entailment but we will construct a multipreference model of a SOM and  use it for validation of strict and preferential properties. 

A ranked $\lc$ knowledge base is a tuple $\langle  {\cal T}_{strict}, {\cal T}_{C_1}, \ldots, {\cal T}_{C_k}, {\cal A}  \rangle$, 
where ${\cal T}_{strict}$ is a set of standard concept and role inclusions, ${\cal A}$ is an ABox and, for each $C_j \in {\cal C}$,  ${\cal T}_{C_j}$ is a ranked TBox containing all the defeasible inclusions, $\tip(C_i) \sqsubseteq D$, specifying the typical properties of $C_i$-elements, with their ranks (non-negative integers). In the following, we will denote ${\cal T}_{C_j}$ as the set 
$\{(d^j_i,r^j_i)\}$, where  each  $d^j_i$ is a typicality inclusion of the form $\tip(C_j) \sqsubseteq D^j_i$,  and $r^j_i$ its rank.
The defeasible inclusions with higher ranks are considered to be more plausible (and hence more important) than the ones with lower ranks.
Let us consider the following example:

\begin{example} \label{exa:horse}
	Consider the ranked KB $\mathit{K= \langle {\cal T}_{strict}, {\cal T}_{Horse}, {\cal T}_{Zebra}, {\cal T}_{Bird}, {\cal T}_{Penguin},  {\cal A} \rangle}$ (with ${\cal A}= \emptyset$), where ${\cal T}_{strict}$ contains the strict inclusions:

$\mathit{Horse \sqsubseteq  Mammal}$  \ \ \ \ \ \  
$\mathit{Mammal \sqsubseteq Animal}$

$\mathit{Zebra \sqsubseteq  Mammal}$    \ \ \ \ \ \  $\mathit{Horse \sqcap Zebra \sqsubseteq  \bot}$  \ \ \ \ \ \  
		
	\noindent
the ranked TBox ${\cal T}_{Horse}=\{(d_1,0),(d_2,1),(d_3, 1),(d_4,2)\}$  contains the defeasible inclusions:  
	
	$(d_1)$ $\mathit{\tip(Horse) \sqsubseteq Has\_Long\_Mane}$ \ \ \ \ \ \ \ \ \ 
	
	$(d_2)$ $\mathit{\tip(Horse) \sqsubseteq Tall }$
	
	$(d_3)$ $\mathit{\tip(Horse) \sqsubseteq RunFast}$  \ \ \ \ \ \ \ \ \  \ \ \ \ \ \ \ \ \  \ \ \ \ \ \ \ \ \ \  
	
	$(d_4)$ $\mathit{\tip(Horse) \sqsubseteq  Has\_Tail}$

\noindent
the ranked TBox ${\cal T}_{Zebra}= \{(d_5,0), (d_6,1), (d_7,2), (d_8,2)\}$ contains the defeasible inclusions:

$(d_5)$ $\mathit{\tip(Zebra) \sqsubseteq RunFast}$

$(d_6)$ $\mathit{\tip(Zebra) \sqsubseteq  \neg Tall}$

$(d_7)$ $\mathit{\tip(Zebra) \sqsubseteq Striped}$

$(d_8)$ $\mathit{\tip(Zebra) \sqsubseteq  Has\_Tail}$

\noindent

	\end{example}
The ranked Tbox ${\cal T}_{Horse}$  can be used to define an ordering among domain elements comparing their typicality as horses.
For instance, given two horses {\em Spirit} and {\em Buddy}, if Spirit has long mane, is tall, has a tail, but does not run fast, it is intended to be more typical than  Buddy, a horse running fast, tall, with long mane, but without tail, as having a tail (rank 2) is a more important property for horses wrt running fast (rank 1). We expect that the ranked knowledge base $K$ above gives rise to multipreference models, where the two preference relations $<_\mathit{Horse}$ and $<_\mathit{Zebra}$ represent the preference among the elements of the domain $\Delta$ according to concepts  $\mathit{Horse}$ and $\mathit{Zebra}$, respectively. We omit the specification of defeasible inclusions for the other distinguished concepts.

Given a DL interpretation $I=\langle \Delta, \cdot^I \rangle$, 
a  modular partial order $<_{C_i}$ can be defined over $\Delta$ for each concept $C_i \in {\cal C}$, where $x <_{C_i} y$ means that $x$ is more typical than $y$ wrt $C_i$ (in the example, $\mathit{Spirit<_{Horse} Buddy}$). 
The definition of the preference relations $ <_{C_1}, \ldots, <_{C_k}$ starting from a ranked knowledge base, can exploit different closure constructions. A lexicographic strategy has been considered in \cite{TPLP2020}, which is one of the strategies considered in  Brewka's framework of basic preference descriptions  \cite{Brewka04} and relates to Lehmann's lexicographic closure construction. However, 
the concept-wise multipreference semantics can better be regarded as a framework  in which alternative notions of preference can be combined \cite{nmr2020}, starting from a modular knowledge base, and alternative preference constructions can been considered for each module, including rational closure, lexicographic closure (as in lexicographic modules \cite{nmr2020}), 
Kern-Isberner's c-representations \cite{Kern-Isberner01,Kern-Isberner2014} or the closure construction introduced for weighted defeasible $\el$ knowledge bases \cite{JELIA2021}.
An algebraic framework for preference combination in Multi-Relational Contextual Hierarchies has  been recently developed by Bozzato et al. \cite{BozzatoIclp2021}.  

In the following, we will assume that the preferences with respect to concepts are given and we reformulate for $\lc$ the notion of concept-wise multi-preference interpretation 
by {\em combining} the preference relations $<_{C_i}$ into a global preference relation $<$. 
This is needed for reasoning about the typicality of arbitrary $\lc$ concepts 
which do not belong to the set of distinguished concepts ${\cal C}$.
For instance,  we may want to verify whether the typicality inclusion $\mathit{\tip(Horse \sqcup Zebra) \sqsubseteq RunFast}$ is satisfied in some multipreference interpretation, i.e., whether the typical instances of concept $\mathit{Horse \sqcup Zebra}$ have the property of running fast.
To check  inclusions of this kind more than one preference relation may be relevant. In this example,  both preference relations $<_\mathit{Horse}$ and $<_\mathit{Zebra}$ are relevant, and they might as well be conflicting for  pairs of domain elements. For instance, if Buddy has all typical properties of a horse and  Marty has all  typical properties of a zebra, Buddy can be more typical than Marty as a horse ($\mathit{Buddy <_\mathit{Horse} Marty}$), but more exceptional as a zebra ( $\mathit{Marty <_\mathit{Zebra} Spirit}$). By {\em combining} the preference relations $<_{C_i}$ into a single {\em global preference} $<$,
one can interpret the concept $\tip(C)$ (the set of typical instances of $C$) as the set of  the minimal $C$-elements with respect to the global preference $<$.
In the case above, neither Marty would be globally preferred to Spirit, nor vice-versa, but both of them should be regarded as typical among $\mathit{(Horse \sqcup Zebra)}$-elements.

In order to define a global preference relation, 
we take into account the specificity relation among concepts,
such as, for instance, the fact that a concept like $\mathit{Baby\_Horse}$ is more specific than concept $\mathit{Horse}$. 
The idea is that,  in case of conflicts, the properties of a more specific class (such as that baby horses normally are not tall) should  override the properties of a less specific class (such as that horses normally are tall).
\begin{definition}[Specificity] \label{specificity}
A {\em specificity relation} among concepts in ${\cal C}$ is a binary relation 
$\succ \subseteq {\cal C} \times {\cal C}$ which is irreflexive and transitive.
\end{definition}
For $C_h, C_j \in {\cal C}$, $C_h \succ C_j$ means that $C_h$ is {\em more specific than}  $C_j$.
The simplest notion of {\em specificity} among concepts with respect to a knowledge base $K$ (an ontology) is based on the subsumption hierarchy: 
 $C_h \succ C_j$ holds
iff $C_h \sqsubseteq C_j$ is subsumed from the ontology, while $C_j \sqsubseteq C_h$ is not.
The notion of specificity based on the subsumption hierarchy has been considered in many approaches to non-monotonic reasoning in DLs, 
including prioritized defaults \cite{baader95b}, prioritized circumscription \cite{BonattiLW06},
and ${\cal DL}^N$  \cite{bonattiAIJ15}. 
Another notion of specificity considered in the literature is the one given by the ranking of concepts in the rational closure \cite{whatdoes} of the knowledge base.

Exploiting the specificity relation among concepts,
the global preference $<$ can be defined by a modified Pareto combination of the relations $<_{C_1}, \ldots,<_{C_k}$,
which takes into account specificity, as follows:
\begin{equation}  \label{def:<}
\begin{aligned}
x <y  \mbox{ iff \ \ } 
(i) &\  x <_{C_i} y, \mbox{ for some } C_i \in {\cal C}, \mbox{ and } \\
(ii) & \ \mbox{  for all } C_j\in {\cal C}, \;  x \leq_{C_j} y   \mbox{ or }  \exists C_h (C_h \succ C_j  \mbox{ and } x <_{C_h} y )
\end{aligned}
\end{equation}
The intuition is that $x< y$ holds if there is at least a concept $C_i \in {\cal C}$ such that $ x <_{C_i} y$ and,
for all other concepts $C_j \in {\cal C}$,   either $x \leq_{C_j} y$ holds or, in case it does not, there is some $C_h$ more specific than $C_j$ such that $x <_{C_h} y$ (i.e., preference  $<_{C_h}$ overrides $<_{C_j}$).
In the example above, for two baby horses (who are also horses) $x$ and $y$, if $\mathit{x <_\mathit{Baby\_Horse} y}$ and $\mathit{y <_\mathit{Horse} x}$, we will have $\mathit{x < y}$, that is, $x$ is regarded as being globally more typical than $y$ 
as it satisfies more properties of typical baby horses  wrt $y$,  
although $y$ may satisfy more properties of typical horses wrt $x$.

We can now formulate the notion of concept-wise multipreference interpretation \cite{TPLP2020} for $\lc$. 
 \begin{definition}[concept-wise multipreference interpretation]\label{def-multipreference-int}  
A {\em concept-wise multipreference interpretation (or {\em cwm}-interpretation)}  is a tuple $\emme= \langle \Delta, <_{C_1}, \ldots,<_{C_k}, <, \cdot^I \rangle$
such that:   
\begin{itemize}
 
\item[(a)]  $\Delta$ is a non-empty domain; 

\item[(b)] for all $i=1,\ldots, k$, $<_{C_i}$ is an irreflexive, transitive, well-founded and modular relation over $\Delta$; 

\item[(c)]  $<$ is the (global) preference relation over $\Delta$ defined from  $<_{C_1}, \ldots,<_{C_k}$ in  (\ref{def:<}); 
\item[(d)]  $\cdot^I$ is an interpretation function, as defined for $\lc$ interpretations 
(see Section \ref{sec:LC}),
with the addition that, for typicality concepts, we let: 
$$(\tip(C))^I = min_{<}(C^I)$$
where $Min_<(S)= \{u: u \in S$ and $\nexists z \in S$ s.t. $z < u \}$.

\end{itemize}
\end{definition}
It has been proven that the relation $<$ is an irreflexive, transitive and well-founded relation \cite{TPLP2020}.
As a consequence, the triple $ \langle \Delta,  <, \cdot^I \rangle$ is a KLM-style preferential interpretation, as those introduced for $\alc$ with typicality \cite{FI09}, while it is not necessarily a modular interpretation. 

The notion of {\em cwm}-model of a ranked KB and the notion of {\em cwm}-entailment can be defined in a natural way for ranked $\lc$ knowledge bases, as it has been done for the lightweight description logic $\elpb$ \cite{TPLP2020}. Let us mention that concept-wise multipreference entailment has been proven to satisfy the KLM postulates of a preferential consequence relation \cite{whatdoes}, and to have good properties such as avoiding the blocking inheritance problem, a well known problem of the rational closure and System Z  \cite{Pearl90,BenferhatIJCAI93},
roughly speaking the problem that, if a subclass of $C$ is exceptional for a given aspect, it is exceptional tout court and does not inherit any of the typical properties of $C$. 
Proof methods for reasoning with ranked and weighted knowledge bases in description logics under the concept-wise multipreference semantics, have been investigated  for lightweight description logics of the $\cal EL$-family \cite{rifel} (namely, for ranked ${\cal EL}^{+}_\bot$ KBs \cite{TPLP2020}  and for weighted $\el$ KBs \cite{ICLP2021}), 
by exploiting Answer Set Programming \cite{GelfondLeone02} and {\em asprin} \cite{BrewkaAAAI15} for defeasible inference.

In the following, 
we will address the problem of defining an $\lc$ multipreference interpretation as a semantic model of a self-organising map. 
Given such a model, the verification of the logical  properties that hold in the SOM can be done by model checking, i.e., by verifying the satisfiability of strict and defeasible inclusions in the model.

\section{Relating self-organising maps and multi-preference models} \label{sec:preference-modelSOM}

We aim at showing that, once the SOM has learned to categorize, we can regard the result of the categorization as a multipreference interpretation. 
Let $X$ be the set of input stimuli from  different categories $C_1, \ldots, C_k$, which have been considered during the learning process.

For each category $C_i$, we let $BMU_{C_i}$ be the set of best-matching units corresponding to  input stimuli of category $C_i$, as in Section \ref{som-general}.
We regard the learned categories $C_1, \ldots, C_k$ as being concept names (atomic concepts) in the description logic (i.e., concept names in $N_C$)  and we let them constitute our set of distinguished concepts ${\cal C}= \{C_1, \ldots, C_k\}$. 
We introduce an individual name $a_x\in N_I$ for each possible stimulus $x$ in the space of the possible stimuli. 

In order to construct a multi-preference interpretation we proceed as follows:
first, we fix the {\em domain} $\Delta^{s}$ to be the space of possible stimuli, that we will assume to be finite and to include $X$, the set of all the input stimuli considered during training;  
then,  for each category (concept) $C_i$,  we define a preference relation $<_{C_i}$ over $\Delta^s$ by exploiting a notion of distance of a stimulus $y$ from the map's representation of $C_i$. Finally, we define the interpretation of concepts.

Let $\Delta^{s}$ be a finite set of possible stimuli, including all input stimuli considered during training ($X \subseteq \Delta^s$) as well as the best matching units of input stimuli (i.e., $\{BMU_x \mid x \in X \} \subseteq \Delta^s$). 
 We therefore build a hybrid model, in which input exemplars and category representations coexist. Notice that we can consider these elements together because they have the same dimensionality. As we will see, category representations will be useful in reasoning about typicality.

Once the SOM has learned to categorize, the notion of distance $d(x,C_i)$ of a stimulus $x$ from a category $C_i$  
introduced above can be used to build  
a binary preference relation $<_{C_i}$ among the stimuli in $\Delta^s$ w.r.t. category $C_i$ as follows:  for all $x, x' \in \Delta^s$,
\begin{align}\label{preferenza_Ci}
x <_{C_i} x' \mbox{\ \  iff \ \ } d(x,C_i) < d(x' ,C_i)
\end{align}
Each preference relation $<_{C_i}$ is a strict partial order relation on $\Delta^s$.
The relation $<_{C_i}$ is also well-founded as we have assumed $\Delta^{s}$ to be finite \footnote{Observe that we could have equivalently used the notion of relative distance  $rd(x,C_i)$ of a stimulus $x$ from a category $C_i$ to define the preference relation, as done in \cite{CILC2020}, rather than using the notion of distance $d(x,C_i)$.}.

We exploit this notion of preference  
to define a concept-wise multipreference interpretation associated to the SOM, and we call it a {\em cwm}-model of the SOM.

\begin{definition}[Concept-wise multipreference-model of a SOM]\label{modello_Som}  
The {\em concept-wise  multipreference model (or {\em cwm}-model) of the SOM} is a cwm-interpretation 
$\emme^{som}= \langle \Delta^{s}, <_{C_1}, \ldots, <_{C_k}, <, \cdot^I \rangle$ 
such that:
\begin{itemize}
\item[(i)] $\Delta^{s}$ is the set of the possible stimuli, as introduced above; 

\item[(ii)]
for each $C_i \in {\cal C}$, $<_{C_i}$ is the preference relation defined by equivalence (\ref{preferenza_Ci});

\item[(iii)] 
$<$ is the global preference relation defined from $ <_{C_1}, \ldots, <_{C_k},$  
as in equation ( \ref{def:<});

\item[(iv)]  
the interpretation function $\cdot^I$ is defined  for individual names $a_x \in N_I$ as $a_x^I=x$,
 and for concept names (i.e. categories) $C_i$ as follows:  
 $$C_i^I= \{y \in \Delta^s \mid d(y,C_i) \leq d_{max,C_i} \}$$
 where $ d_{max,C_i}$   is the maximal distance of an input stimulus $x \in C_i$ from category $C_i$, that is,
$d_{max,C_i} = max_{x \in C_i} \{d(x, C_i)\}$. 
The interpretation function $\cdot^I$ is extended to complex concepts 
according to Definition \ref{def-multipreference-int}, point (d).

\end{itemize}
\end{definition}
Informally, we interpret  as $C_i$-elements the stimuli whose 
distance from category $C_i$ is not larger than the
distance of any input exemplar belonging to category $C_i$\footnote{As we can see, roles are not necessary to define a preferential model of a SOM. This is the reason why we have not introduced them in the language, although they were present (but not  used) in the preliminary version of the paper \cite{CILC2020}, exploiting the description logic $\el$. On the other hand, here, we have preferred to considered a logic with union and complement operators,  which are not present  in $\el$ but useful for the specification of the properties to be verified.}.

Given $<_{C_i}$, we can identify the most typical $C_i$-elements  (wrt $<_{C_i}$)  
as the $C_i$-elements  in $min_{<_{C_i}}(C_i^I)$, whose  distance from category $C_i$ is minimal.
Observe that 
 the best matching unit $BMU_x$ of an input stimulus $x \in C_i$ is an element of $\Delta^s$.
Hence, for $y=BMU_x$,
the   distance $d(y,C_i)$ of $y$ from category $C_i$ is $0$, as  $min \|  y - BMU_{C_i} \| =0$.  
Therefore, 
$min_{<_{C_i}}(C_i^I) =\{ y \in \Delta^s \mid \; d(y,C_i)=0\}$ and $BMU_{C_i} \subseteq min_{<_{C_i}}(C_i^I)$.
The converse inclusion $BMU_{C_i} \supseteq min_{<_{C_i}}(C_i^I)$ might not hold.  In fact, in case there is some input exemplar $x \in C$ 
whose weight vector exactly coincides with the weight vector of its best matching unit $BMU_x$, then both $x$ and $BMU_x$ belong to $min_{<_{C_i}}(C_i^I)$.  Notice that $x$ and $BMU_x$ are different elements of $\Delta^s$ ,  
even when they are associated to the same vector.

In $\emme^{som}$, as in all {\em cwm}-interpretations (see Definition  \ref{def-multipreference-int}), the interpretation of typicality concepts $\tip(C)$ is defined based on the global preference relation $<$ as $(\tip(C))^I= min_<(C^I)$, for all  concepts $C$. The model $\emme^{som}$ can be considered a sort of canonical model representing what holds in the SOM after the learning phase.
Note that, for all domain elements $y \in \Delta^s$ which have not been considered during training and do not correspond to some best matching unit, whether $y \in C_i^I$ for a concept $C_i$ and whether $y <_{C_i} x$, for $x \in \Delta^s$, is determined based on the distance $d(y,C_i)$ of $y$ from category $C_i$, as for the input stimuli $x \in X$ (see points $(iii)$, $(iv)$ in Definition \ref{modello_Som}).

As $\emme^{som}$ is a {\em cwm}-interpretation, the triple $ \langle \Delta^{s}, <, \cdot^I \rangle$ is a KLM style preferential interpretation  \cite{KrausLehmannMagidor:90,whatdoes}. It follows that $\emme^{som}$ provides a preferential semantics of the SOM.
The preferential interpretation $\emme^{som}$ determines a preferential consequence relation (the set of all conditionals $\tip(C)\sqsubseteq D$ true in $\emme^{som}$), which satisfies all KLM postulates of a preferential consequence relation \cite{KrausLehmannMagidor:90}.

\subsection{Evaluation of strict and defeasible concept inclusions by model checking} \label{sec:model_checking}

In the previous section, we have defined a multipreference interpretation $\emme^{som}$ from a SOM where we are able to identify,  on the domain $\Delta^{s}$ of the possible stimuli,  the set of $C_i$-elements as well as the set of  most typical $C_i$-elements wrt $<_{C_i}$, for each category $C_i$.
Provided the necessary information are recorded during training, 
we are then able to evaluate strict inclusions $C \sqsubseteq D$ and defeasible inclusions  $\tip(C) \sqsubseteq D$, 
by {\em model checking},  i.e., by verifying their satisfiability in the model $\emme^{som}$. 

For instance, we may want to check whether elephants are African or Asiatic animals (i.e., if $\mathit{Elephant \sqsubseteq African\_Animal \sqcup Asian\_Animal}$ holds in the model), or whether typical elephants are African or Asiatic animals
(i.e., $\mathit{\tip(Elephant) \sqsubseteq African\_Animal \sqcup }$ $\mathit{Asian\_Animal}$), 
or whether typical big elephants are African
(i.e., $\mathit{\tip(Elephant \sqcup Big\_}$ $\mathit{Animal) \sqsubseteq }$ $\mathit{ African\_Animal}$), 
provided the concepts $\mathit{Elephant}$, $\mathit{African\_Animal}$,  $\mathit{Asi}$- $\mathit{atic\_Animal}$ and $\mathit{Big\_Animal}$ correspond to learned categories.
The latter inclusions represent weaker properties with respect to the first one as they are only concerned with the typical instances of concepts. A strict inclusion $\mathit{Elephant \sqsubseteq Big\_Animal}$, for instance, cannot be expected to be true in the model of the SOM, if  input stimuli also include small elephants such as baby elephants, while a defeasible inclusion $\mathit{\tip(Elephant) \sqsubseteq Big\_Animal}$ might hold.

The verification of general concept inclusions, where $C$ and $D$ are $\lc$ concepts, i.e.,  boolean combinations of concepts in ${\cal C}$, may require to record the extensions of  concepts $C_i$'s in the model $\emme^{som}$ and to perform set-theoretic operation to compute the interpretation of concepts. For instance, concept $(C_1 \sqcap C_2) \sqcup \neg C_3$ is interpreted as $(C_1^{I} \cap C_2^{I}) \cup (\Delta^s \backslash C_3^{I})$. Finding  typical instances of this concept would further require to compute the global preference relation $<$ among its instances and, in the worst case, to verify for all pairs of input stimuli $x,y \in \Delta^s$ whether $x< y$ or $ y < x$ or $x$ and $y$ are incomparable. This, in turn, requires to check, for all categories $C_i$, whether $x <_{C_i} y$ or $y <_{C_i} x$ or $x \sim_{C_i} y$, based on the  distances $d(x,C_i)$ and $d(y,C_i)$.
This might be challenging in practice, depending on the size of the domain, 
 although the verification requires a polynomial number of checks. 
In the following of this section, we will let $b$ to be the largest number of best matching units for a  category. Note that $b$ cannot be larger than the size of the set $X$ of all input stimuli.

\begin{proposition}
The verification that a general typicality inclusion is satisified in the multipreference interpretation $\emme^{som}$ is $O(n^3  \times k)$, where $n$ is the size  of $\Delta^s$ and  $k$ the number of categories. 
\end{proposition}
\begin{proof} 
Consider a typicality inclusion $\tip(C) \sqsubseteq D$, where $C$ and $D$ are general $\lc$ concepts.

Observe that computing $d(x,C_i)$ for some $x \in \Delta^s$ and concept $C_i \in {\cal C}$ requires  to compute the distance of $x$ from all the best matching units in $BMU_{C_i}$. This requires a number of steps equal to the size of $BMU_{C_i}$. Hence, computing $d(x,C_i)$ is $O(b)$.

Verifying  whether $x \in C_i$, for $x \in \Delta^s$ and concept $C_i \in {\cal C}$, amounts to check whether $d(x,C_i) \leq d_{max,C_i}$, and is $O(b)$, if $d_{max,C_i}$ has been computed during training.
Verifying whether $x <_{C_i} y$ for $x, y \in \Delta^s$ and $i=1,\ldots,k$, requires to check whether $d(x,C_i)  < d(y,C_i) $,  
and is $O(b)$ (and, similarly, verifying  $x \sim_{C_i} y$).
Checking $x < y$, requires  to verify whether $x <_{C_i} y$ and whether $x \sim_{C_i} y$ for all $C_i \in {\cal C}$.  Hence, checking $x < y$ for some $x, y \in \Delta^s$ is $O(b \times k)$. 

Identifying the instances of a given boolean concept $C$ requires to verify, for each domain element $x \in \Delta^s$, if $x \in C^I$. 
It requires to check whether $x$ belongs to $C_i$, for all concepts $C_i$ occurring in $C$.
As checking whether $x$ belongs to $C_i$ is $O(b)$, verifying this for all $C_i$ occurring in $C$ is $O(b \times k)$.

Once we know whether $x \in C_i^I$ is known for all $i$'s, determining whether $x \in C^I$
requires to evaluate a boolean expression, which is linear in the size of the expression $C$ (that we assume to be $O(k)$). Overall, verifying whether $x \in C^I$ is $O(b \times k^2)$.
Thus, identifying all $C$-elements in the domain $\Delta^s$ is $O(n \times b \times k^2)$ and, as $b$ cannot be larger than $n$, it is $O(n^2 \times k^2)$.

Let $S$ be the set of all $C$-elements. Identifying the $<$-minimal $C$-elements, among all $C$-elements, can be done in $O(n^2)$ steps,  starting from a set $S$ (whose size may be comparable to the size of $\Delta$), initialized to contain all $C$-elements, by iterating on all $C$-elements $y$ and by removing from $S$  all elements $y \in S$ such that $y' <y$ for some $y' \in \Delta$ (using a two nested loop).
Verifying $y' <y$ takes constant time only if the distances  $d(x,C_i)$, for all $x \in \Delta$ and all $C_i$, are computed and stored in advance, which is unlikely when the size of $\Delta^s$ is large. Otherwise, verifying $y' <y$ is $O(b \times k)$, as seen above. In this second case,  identifying the $<$-minimal $C$-elements, among all $C$-elements is $O(b\times k \times n^2)$, and hence $O(n^3 \times k)$.
 
For each $x \in min_{<}(C^I)$, verifying that $x \in D^I$ (as for $C$ above) is  $O(b \times k^2)$.
The size of the set $min_{<}(C^I)$ is smaller than the size of $\Delta^s$, hence verifying, for all $x \in min_{<}(C^I)$, that $x \in D^I$
is $O(n \times b \times k^2)$ and, hence,  $O(n^2 \times k^2)$.
 
We then have a sequence of steps, where each step is either $O(n^2 \times k^2)$ or $O(n^3 \times k)$. As $n \gg k$, 
verifying that a general typicality inclusion is satisified in $\emme^{som}$ is $O(n^3  \times k)$ 
 \qed
\end{proof} 
We have seen that, in the general case, verifying the satisfiability of an inclusion on the model  of the SOM may be non trivial, depending on the number of input stimuli that are considered in the learning phase (the size of the set $X$ of input exemplars and their best matching units) and on the size of $\Delta^s$. 
Gliozzi and Plunkett have  considered self-organising maps that are able to learn from a limited number of input stimuli, although this is not generally true for all self-organising maps \cite{CogSci2017}. 
In the following we will see that, however, verifying the satisfiability of strict inclusions $C_i \sqsubseteq C_j$ and defeasible inclusions $\tip(C_i) \sqsubseteq C_j$ where $C_i$ and $ C_j$ are categories (i.e. distinguished concepts in ${\cal C}$), requires a  
smaller number of steps, 
as it  does not require
to compute the distance of each stimulus from each concept representation.

In order to verify that a typicality inclusion $\tip(C_i) \sqsubseteq C_j$ is satisfied in $\emme^{som}$, we have to check that the most typical $C_i$ elements wrt $<_{C_i}$ are $C_j$ elements, that is,  $min_{<_{C_i}}(C_i^I) \subseteq C_j^I$. 
Note that, besides the elements in $BMU_{C_i} $, $min_{<_{C_i}}(C_i^I)$ might contain other elements of $\Delta^{s}$ having distance $0$ from $C_i$, namely those elements $y$ from the domain  
having distance $0$ from their best matching unit (elements whose weight vector exactly coincides with the weight vector of a best matching unit $BMU_x$ for $C_i$). In such a case, as $BMU_x$ is already in $BMU_{C_i} $, it is enough
 to verify that all elements in  $BMU_{C_i} $ 
are $C_j$-elements, that is: 
\begin{equation} \label{cond_typ_incl1}
\mbox{  for all  input stimuli $x \in C_i$, } d(BMU_x, C_j) \leq d_{max,C_j}
\end{equation}
that is, the distance of all best matching units of $C_i$ from $C_j$ is  smaller than the distance from $C_j$ of some $C_j$-element in the input set.
Let the distance of $BMC_{C_i}$ from $C_j$ be defined as follows:
$$d( BMC_{C_i},C_j) = max_{x \in C_i} \{d(BMU_x, C_j)\}$$
as the maximal distance of  any $BMU_{x}$, for $x \in C_i$, from $C_j$.  
 Then we can rewrite condition (\ref{cond_typ_incl1}) simply  as
 \begin{equation} \label{cond_typ_incl}
d( BMC_{C_i},C_j)  \leq  d_{max,C_j}.
\end{equation}
Observe that the distance $d( BMC_{C_i},C_j)$  provides a measure of plausibility of the defeasible inclusion $\tip(C_i) \sqsubseteq C_j$: the higher is the  distance of $BMU_{C_i}$ from $C_j$, the more implausible is the defeasible inclusion $\tip(C_i) \sqsubseteq C_j$. The lower is $d( BMC_{C_i},C_j)$, the more plausible is the defeasible inclusion $\tip(C_i) \sqsubseteq C_j$. This is relevant if we aim at extracting knowledge in the form of a set of conditionals from a SOM.
In particular,  using  the relative distance and the generalization degree  
in \cite{CogSci2017}, the value  $e ^{- rd( BMC_{C_i},C_j)}$ provides a degree of plausibility of a defeasible inclusion $\tip(C_i) \sqsubseteq C_j$ in the interval $(0,1]$.

Let us now consider the case of a strict inclusion $C_i \sqsubseteq C_j$, with $C_i, C_j \in {\cal C}$.
Verifying that $C_i \sqsubseteq C_j$ is satisfied by $\emme^{s}$, requires to check that $C_i^I$ is included in $C_j^I$. 
Exploiting the fact that the map is organized topologically, and using the distance $d( BMC_{C_i},C_j)$ of $BMC_{C_i}$ from $C_j$ defined above,
we can verify that the distance of $BMC_{C_i}$ from $C_j$ plus the maximal distance of a $C_i$-element from  $C_i$ is not greater than the maximal distance of a  $C_j$-element from $C_j$:
\begin{equation} \label{cond_strict_incl}
d( BMC_{C_i},C_j) + d_{max,C_i} \leq d_{max,C_j},
\end{equation}
that is, the most distant  $C_i$-element  from $C_j$ is nearer to $C_j$ than the most distant $C_j$-element.

Note that the verification of conditions (\ref{cond_typ_incl}) and  (\ref{cond_strict_incl}) does not require to record all the instances of the concepts $C_i$ and $C_j$ in $\Delta^s$, but to compute and record some measures, namely $d_{max,C_j}$ and  $d( BMC_{C_i},C_j) $, for all concepts (categories) $C_i, C_j \in {\cal C}$.
While computing $d_{max,C_j}$ during training only requires a constant overhead (by updating the current value of  $d_{max,C_j}$ for each input stimulus in $C_j$), computing $d( BMC_{C_i},C_j) $ requires to compute the maximum distance between the best matching units for $C_i$ and the  best matching units for $C_j$, requiring $O(b^2)$ steps, where $b$ is the number of best matching units for the input stimuli.
Once such measures have been computed,  
the verification  of conditions (\ref{cond_typ_incl}) and  (\ref{cond_strict_incl}) requires constant time. 
The next proposition follows:
\begin{proposition}
Checking the satisfiability of a strict  (resp., defeasible) inclusion of the form $C_i \sqsubseteq C_j$ (resp., $\tip(C_i) \sqsubseteq C_j$) in the interpretation $\emme^{som}$ 
 is  $O(b^2)$, 
 when $C_i$ and $C_j$ are distinguished concepts in ${\cal C}$ corresponding to learned categories.
\end{proposition}
 Notice that the number of the best matching units for $C_i$ is smaller than the size of $X$ and of the domain $\Delta^s$. Furthermore, 
as observed by Gliozzi and Plunkett, the number of best matching units may sometimes be significantly smaller than the set of the input stimuli. 
Even when the domain $\Delta^s$ only contains the input stimuli in $X$ and their best matching units,  
the verification of general typicality formulas over the SOM model   is $O(n_X^3  \times k)$, where $n_X$ is the size of set of input stimuli $X$ considered during training. It may be challenging in practice, due to the large number of input stimuli,  and more challenging than validating inclusions of the form  $C_i \sqsubseteq C_j$ and $\tip(C_i) \sqsubseteq C_j$, as $b$ may be significantly lower than $n_X$.
An alternative approach for reasoning about the knowledge learned by the SOM might be:
 first identify 
the set $K$ of strict and defeasible inclusions of the form $C_i \sqsubseteq C_j$ and $\tip(C_i) \sqsubseteq C_j$ satisfied in $\emme^{som}$  (where $C_1, \ldots, C_k$ are the learned categories); then exploit the defeasible knowledge base $K$ {\em extracted} from the SOM for {\em symbolic reasoning} in a conditional logic formalism.

For instance, as we have already mentioned, an approach which exploits Answer Set Programming and {\em asprin} \cite{BrewkaAAAI15} to achieve defeasible reasoning under the concept-wise multipreference semantics
has been proposed for reasoning with  ranked  $\elpb$ knowledge bases \cite{TPLP2020}.
The approach has been extended to deal with weighted  knowledge bases with real valued weights \cite{ICLP2021}.
As we have mentioned above, a measure of plausibility can be determined for the defeasible inclusions satisfied by the SOM, and it can be exploited for the definition of a weighted knowledge base. 
Other proposals, which are not based on multiple preferences, have been developed in the literature of non-monotonic and conditional reasoning  for dealing with weighted knowledge bases. Let us mention Weydert's System JLZ  \cite{Weydert03}, an approach which allows rational and real-valued ranking measures, and the work by Kern-Isberner and Eichhorn \cite{Kern-Isberner2014}, based on c-representations which also allow for plausibility weights.

\subsection{Dealing with specificity} \label{sec:specificity}

Let us conclude this section by commenting on the notion of specificity, that we have left aside up to this point.
Reasoning about specificity is an important feature of a non-monotonic formalism, when it is intended to deal with exceptions among classes.
This was recognized from the beginning by 
Baader and Hollunder  \cite{baader95b} who, in their work on prioritized defaults in description logics,  observe that
\begin{quote}
\em``the question of {\em how to prefer more specific defaults over more general} ones
[...]  is of general interest for default reasoning but is {\em even more important in the terminological case where the emphasis lies on the hierarchical organization of concepts}".
\end{quote}
Many non-monotonic extensions of description logics,  
including the ones based on circumscription \cite{BonattiLW06}, on the rational closure \cite{casinistraccia2010,AIJ15}, on the lexicographic closure \cite{Casinistraccia2012} and their refinements, conform to the principle that the specificity relation among concepts is to be taken into account. 

To see that specificity is also relevant in this context,
let us suppose that, among the categories $C_1, \ldots, C_k$, both the category Bird and the category Penguin are present.
Let us further suppose that, as a result of the categorization, the inclusion $\mathit{Penguin \sqsubseteq Bird}$ is satisfied in the model $\emme^{som}$.
Clearly, the preference relations  $\mathit{<_{Penguin}}$ and $\mathit{<_{Bird}}$ might not agree as typical penguins do not fly and are atypical as birds. In this case, the class $\mathit{Penguin}$ is more specific than the class $\mathit{Bird}$. 
 As we expect that typical $\mathit{Penguin \sqcap Bird}$-elements are penguins and hence do not fly, we would expect the typical instances of  concept $\mathit{Penguin \sqcap Bird}$, i.e., $\mathit{min_< (Penguin \sqcap Bird)}$, should correspond to the minimal elements with respect to $\mathit{<_{Penguin}}$, i.e., to $\mathit{min_{<_{Penguin}} (Penguin \sqcap Bird)}$. 
 
In Section  \ref{sec:multipref} we have defined the global preference relation in such a way that the specificity relation among concepts is taken into account.
The idea is that,  in case of conflicts, the properties of a more specific class, $\mathit{{Penguin}}$, should  override the properties of a less specific class, $\mathit{{Bird}}$, as $\mathit{{Penguin \succ Bird}}$ (concept $\mathit{{Penguin}}$ is more specific than $\mathit{{Bird}}$).
 
The definition of global preference allows us to deal with the specificity issues when they emerge, that is, when a category $C_i$ is subsumed by another $C_j$, as a result of categorization.
In such a case, we expect the subsumption $C_i \sqsubseteq C_j$ to hold in the model of the SOM $\emme^{som}$, whereas we do not expect the converse inclusion to hold. 

In other cases, the SOM might not recognize that concept $C_i$ is more specific than $C_j$, 
but we may know that it is the case from some alternative knowledge source (e.g., an ontology) rather than from the analysis of empirical data in the SOM model. 
In such a case, we can nevertheless exploit the specificity information $\mathit{Penguin \succ Bird}$ in the construction of the global preference relation $>$ from the preferences $<_{C_j}$ built from the SOM, or we can add the specificity inclusions to the defeasible knowledge base containing the strict and defeasible inclusions extracted from the SOM. This is a possible way of exploiting symbolic knowledge in combination with the knowledge extracted from empirical data, i.e., from the preferential model of SOM.

In this section, we have used the relative distance of a stimulus $x$ from a category $C_i$ to define a preference relation $<_{C_i}$ on the set of input stimuli. This preference relation determines the typical $C_i$-elements, the domain elements with the lowest relative distance from category $C_i$. 
In the next section, we will exploit the same notion of relative distance of a stimulus $x$ from a category $C_i$ to define a {\em degree of membership} of $x$ to category $C_i$. We will do this by associating  a fuzzy $\lc$ interpretation to a SOM.

\section{A fuzzy interpretation of a Self-Organising Maps} \label{sec:fuzzy_SOM}

In this section, we consider an alternative interpretation of Self-Organising Maps based on fuzzy description logic interpretations. 
As mentioned in Section \ref{sezione:fuzzyDL},
fuzzy description logics allows for representing vagueness in DLs and have been widely studied in the literature  \cite{Straccia05,Stoilos05,LukasiewiczStraccia09,PenalosaARTINT15,BobilloStracciaEL18}. 
In fuzzy description logics, a concept $C$ is interpreted as 
a function $C^I$  mapping each domain element $x$  to a  value in the interval $[0,1]$, a degree of membership in $C$.

We again follow Gliozzi and Plunkett's  similarity-based account of category generalization \cite{CogSci2017}  (see Section \ref{som-general}) where,
from the notion of relative distance $rd(x,C_i)$ of an input stimulus $x$ from category $C_i$, they define  
the {map's Generalization Degree} of category $C$ membership to a new stimulus $y$  (here denoted by $gd_C(y)$),  as:
$gd_C(y)= e^{-rd(y,C)} $.
The value of 
$gd_C(y,C)$ is in $(0,1]$ and can be regarded as the degree of membership of $y$ in $C_i$. 

As in the definition of the multipreference model $\emme^{som}$ in Section  \ref{sec:preference-modelSOM}, we let the domain $\Delta^s$ be the set of all the possible stimuli, including the input stimuli considered in the training phase and their best matching units in the map. We also consider the learned categories $C_1, \ldots, C_k$ as the concept names (atomic concepts) in the description logic, i.e., $N_C= \{C_1, \ldots, C_k\}$.
We let $N_I$ be the set of individual names as in previous section.

\begin{definition}[Fuzzy model of a SOM]\label{modello_fuzzy_Som}  
The {\em fuzzy model of the SOM} is a fuzzy interpretation 
 $I^{s} = \langle \Delta^s, \cdot^I \rangle$ 
such that:
\begin{itemize}
\item $\Delta^{s}$ is the set of possible stimuli, as introduced above; 

\item
the interpretation function $\cdot^I$ is defined  for individual names $a_x \in N_I$ as $a_x^I=x$, 
 and  for all named concepts $C_i$ as: $C_i^I(x) = e^{-rd(x,C_i)}$.
\end{itemize}
\end{definition}
With this definition, the fuzzy interpretation $I^{s}$ interprets each concept $C_i$ corresponding to a learned category as a fuzzy set.
As recalled in Section \ref{sezione:fuzzyDL}, the fuzzy interpretation of complex concepts is defined inductively, given a choice of t-norm, s-norm, implication function, and negation function.

Observe that, in the fuzzy model of the SOM, the (normalized) relative distance $rd(x,C_i)$ of $x$ from $C_i$ is used to define the value of $C_i^I(x)$, rather than the distance $d(x,C_i)$. 
As $rd(x,C_i) \leq 1$ for all instances $x$ of $C_i$, for all instances $x$ of a category $C_i$, it holds that $C_i^I(x) \leq e^{-1}$.  
For $y \in \Delta^s$, the value of $C_i^I(y)$ is $1$ when $rd(y,C_i)=0$ and  approaches $0$ when $rd(y,C_i)$ approaches $\infty$.

As we have seen in section \ref{sezione:fuzzyDL}, in the fuzzy interpretation $I^{s}$, a concept inclusion axiom $C_i \sqsubseteq C_j$ is interpreted as 
$(C_i \sqsubseteq C_j)^{I^{s}}= inf_{x \in \Delta}  \; C_i^{I^{s}}(x) \rhd C_j^{I^{s}}(x)$, which describes the degree of subsumption between the two fuzzy sets  $C_i^{I^{s}}$ and  $C^{I^{s}}_j$. 
Computing $(C_i \sqsubseteq C_j)^{I^{s}}$ requires the values $ C_i^{I^{s}}(x)$ and $ C_j^{I^{s}}(x)$ to be computed or 
recorded for all domain elements $x$, that is,  for all the stimuli considered in $\Delta^s$. The satisfiability of a fuzzy inclusion axiom such as $C_i \sqsubseteq C_j \;\geq\; n$ can then be evaluated in the fuzzy model $I^{s}$, by verifying that 
$(C_i \sqsubseteq C_j)^{I^{s}} \geq n$.
Although conceptually this is just model checking, it might be challenging in practice, depending on the size of the domain $\Delta^s$.

The possibility of verifying the satisfiability of fuzzy axioms over the model $I^{som}$ of the SOM allows for symbolic knowledge to be extracted from the SOM
in the form of fuzzy inclusion axioms. This knowledge can then be used for symbolic reasoning, based on the proof methods that have been developed for both  expressive and lightweight fragments of fuzzy OWL \cite{BobilloStraccia16,BobilloStracciaEL18}.  
The problem of learning fuzzy rules has been widely investigated in the context of fuzzy description logics 
based on other machine learning approaches \cite{LisiStraccia2015,StracciaMucci2015}.

For instance, referring to the running example,  we may want to check whether  elephants are big animals with a degree $\geq 0.6$, i.e., whether the fuzzy axiom $\mathit{Elephant \sqsubseteq }$ $\mathit{ Big\_Animal} \geq 0.7$ is satisfied in the interpretation $I^s$, and whether the fuzzy axioms
$\mathit{Elephant \sqcup Hyppos \sqsubseteq Big\_Animal} \geq 0.7$   and
$\mathit{Elephant \sqcap White\_Animal \sqsubseteq Afri}$- $\mathit{can\_Animal \sqcup Asian\_Animal} \geq 0.8$  hold in $I^s$.

Notice that, while the verification of fuzzy axioms on the fuzzy model $I^s$ depends 
 on all domain elements in $\Delta^s$,  the verification on the preferential model $\emme^s$ of the prototypical properties of a category $C_i$ by an inclusion $\tip(C_i) \sqsubseteq C_j$ only requires to  consider the typical $C_i$-elements, i.e., the elements in $BMU_{C_i}$, rather than all $C_i$-elements.
 Indeed, although the preferential and the fuzzy semantics of the SOM are constructed starting from the measures of  
distance of the input stimuli from the categories, the two semantics 
may provide quite different information about the SOM, which are expressed, respectively in the form of conditional assertions and of fuzzy concept inclusions.  
For instance, in the presence of high variability in the examples of a category, typical elements do not necessarily provide a precise representation of the category and of the properties of all its elements 
(we will return to the variability effect in the next section).
Although a typicality inclusion $\tip(C) \sqsubseteq D$ may be satisfied, as the typical $C$-elementas are $D$-elements, it might not be the case that the value of $(C \sqsubseteq D)^I$ is high (i.e., that it approaches 1), as there might be  elements in $C$ which are very exceptional with respect to property $D$.
The presence in the domain $\Delta^s$ of additional stimuli with respect to the input stimuli in $X$ (and their best matching units), may affect the valuation $(C \sqsubseteq D)^I$ of an inclusion $C \sqsubseteq D$ in the fuzzy interpretation $I^s$. This is the same for general typicality inclusions in the multipreference interpretation $\emme^{som}$ but, as mentioned above, this is not the case for typicality inclusion $\tip(C_i) \sqsubseteq C_j$, with $C_i$ and $C_j$ distinguished concepts in ${\cal C}$, whose evaluation only depends on the prototypical elements in the category $C_i$, i.e., on the best matching units for category $C_i$.

\normalcolor

\section{Towards a probabilistic interpretation of Self-Organising maps} \label{sec:probab}

In previous section, we have defined a fuzzy DL interpretation $I^{s} = \langle \Delta^s, \cdot^I \rangle$  as the semantics of a SOM after the training phase.
Each concept $C$ in  
$I^s$ is interpreted as a fuzzy set on $\Delta$ with membership function $C^I$.
The relationships between fuzzy sets and probability theory have been widely investigated in the literature \cite{Zadeh1968,Kosko90,DuboisP1993_fuzzy_probab},
and we will follow the proposal by Zadeh, who has shown that 
"the notions of an event and its probability can be extended in a natural fashion to fuzzy events" \cite{Zadeh1968}. In particular, we refer to a recent characterization of the continuous t-norms compatible with Zadeh's probability of fuzzy events ($P_Z$-compatible t-norms) by Montes et al. \cite{Montes2013},  to provide a probabilistic interpretation starting from the fuzzy interpretation $I^s$. We will see that, 
following this approach, 
we can regard $C^I(x)$ as the conditional probability of the (fuzzy) concept $C$, given element $x$.
This is in agreement with Kosko's
account of "fuzziness in a probabilistic world" \cite{Kosko90}. 

Zadeh has proven that the set of fuzzy events forms a $\sigma$-field with respect to the operations of complement, union and intersection in (Zadeh's) fuzzy logic.
Montes et al. \cite{Montes2013} have  
studied the problem of determining which t-norms make Zadeh's probability of fuzzy events fulfill Kolmogorov's axioms 
(called $P_Z$-compatible t-norms). 
For a given t-norm $T$, they consider a generalization of the notion of algebra to the fuzzy framework by means of T-clans.

 For a given universe $\Omega$, they let ${\cal F}(\Omega)$ denote the set of all fuzzy subsets on $\Omega$,
and a T-clan ${\cal A}$ be a subset of ${\cal F}(\Omega)$ satisfying the properties that:
\begin{itemize}
\item $\emptyset \in {\cal A}$,  where $\emptyset(x)=0$ for all $x \in \Omega$;
\item
if $A \in {\cal A}$, then $A^c \in {\cal A}$, where $A^c(w)= 1-A(w)$;
\item
if $A , B \in {\cal A}$, then $A \cap_{T} B \in {\cal A}$.
\end{itemize}
Given the universe $\Omega$, the set ${\cal F}(\Omega)$ of all fuzzy subsets on $\Omega$, and a probability measure $P$ over $\Omega$,
they define the probability of a measurable fuzzy event $A$ in a T-clan ${\cal A}$ on ${\cal F}(\Omega)$ 
as 
\begin{equation}\label{eq_Prob}
P(A)= \int_{\Omega} A(w)  dP(w).
\end{equation}
Montes et al.  \cite{Montes2013}  extend the result by Zadeh by identifying a class of $P_Z$-compatible t-norms, for which  the probability $P$ defined by equation (\ref{eq_Prob}) satisfies Kolmogorov's axioms (formulated in the fuzzy framework).  For a t-norm $T$, they consider the t-conorm $S_T(x,y)= 1- T(1-x,1-y)$.

Let us restrict to a $P_Z$-compatible t-norm $T$, with associated  t-conorm $S_T$ and the negation function $\ominus x= 1-x$.
For instance, we can assume the combination functions in Zadeh logic or in Lukasiewicz logic.
Observe that in the interpretation $I^{s}=\langle \Delta^s, \cdot^I \rangle$ constructed from the SOM, the domain $\Delta^s$ is a finite set of input stimuli 
and, based on the chosen t-norm $T$, each concept $C$ is interpreted as a fuzzy set on $\Delta^s$ with  membership function
$C^I: \Delta^s \ri [0,1]$ (a measurable function, being $\Delta$ finite).
Let  ${\cal F}(\Delta^s)$ denote the set of all fuzzy subsets on $\Delta^s$. 
Given $I^{s}=\langle \Delta^s, \cdot^I \rangle$, the collection ${\cal A}$ of the fuzzy sets $C^I$, for all DL concepts $C$, forms a $T$-clan on ${\cal F}(\Delta^s)$. 

By assuming a discrete probability distribution $p$ over $\Delta^s$, 
we can define the probability of the fuzzy set $C^I$, for each DL concepts $C$ as:
\begin{align}\label{eq_probab}
P(C^I)=\sum_{d \in \Delta^s} C^I(d) \; p(d)
\end{align}
As the interpretation $I^s$ is the unique fuzzy interpretation associated to the SOM, we can simply write $P(C)$, rather than $P(C^I)$ and read
$P(C)$ as the probability of a fuzzy concept $C$, given a probability distribution $p$ over the domain $\Delta^s$ of the input stimuli.

For a given input stimulus $x \in \Delta^s$, the probability the $x$ is an instance of a fuzzy concept $C$, can then be interpreted as the conditional probability   
$P(C \mid x) $ of $C$ given $x$ (where $x$ stands for the crisp set $\{x\}$). 
Following Smets \cite{Smets82}, we let the conditional probability of a fuzzy event $C$ given the fuzzy event $D$ be
\begin{align*}
P(C \mid D)=\frac{P(D \sqcap C) }{ P(D)}
\end{align*}
(provided $P(D)>0$) and depart from Zadeh's definition of conditional probability, which exploits the product of the fuzzy sets rather than the intersection.
As observed by Dubois and Prade  \cite{DuboisP1993_fuzzy_probab}, letting $P(G \mid F)= P(F \cap G)/ P(F)$   (for the fuzzy events $F$ and $G$),
generalizes both conditional probability 
and the fuzzy inclusion index advocated by Kosko \cite{Kosko92}. 

From the definition above of conditional probability and from condition (\ref{eq_probab}), it easily follows that $P(C \mid x) = C^I(x)$. In fact:

\begin{align*}
P(C \mid x)= \frac{ P((C \sqcap \{x\})^I) }{ P(\{x\}^I)} & =  \frac{ \sum_{d \in \Delta} (C \sqcap \{x\})^I(d) \; P(d) } {\sum_{d \in \Delta}  \{x\}^I(d) \; P(d) } 
\end{align*}
As $(C \sqcap \{x\})^I(d)= 0$ for $d\neq x$, and  $ \{x\}^I(d)=0$ for $d\neq x$, this simplifies to:
\begin{align*}
 &  \frac{  (C \sqcap \{x\})^I(x) \; P(x) } {  \{x\}^I(x) \; P(x) }
\end{align*}
and, observing that $(C \sqcap \{x\})^I(x)= C^I(x)$ and  $ \{x\}^I(x)=1$:
\begin{align*}
P(C \mid x) &  
=  C^I(x)
\end{align*}

Furthermore, if we assume a uniform probability distribution $P$ (i.e., $P(x)=1/n$, for all $x \in \Delta^s$,  where $n=|\Delta^s|$), it holds that
$$P(x|C)= C^I(x) / M(C)$$
where 
$M(C)= \sum_{x \in \Delta^s} C^I(x)$ is called the {\em size} of the fuzzy concept $C^I$. In fact,
\begin{align*}
P(x \mid C)= \frac{ P((C \sqcap \{x\})^I) }{ P(C^I)} & =  \frac{ \sum_{d \in \Delta} (C \sqcap \{x\})^I(d) \; P(d) } {\sum_{d \in \Delta}  C^I(d) \; P(d) } 
\end{align*}
As $(C \sqcap \{x\})^I(d)= 0$ for $d\neq x$ and $(C \sqcap \{x\})^I(x)= C^I(x)$,  this simplifies to:
\begin{align*}
 &    \frac{  C^I(x) \; P(x) } {\sum_{d \in \Delta}  C^I(d) \; P(d) } 
\end{align*}
and, assuming  a uniform probability distribution over $\Delta$: 
\begin{align*}
P(x \mid C) & =   \frac{  C^I(x) \;(1/n) } {(1/n) \; \sum_{d \in \Delta}  C^I(d) }  =   \frac{  C^I(x) } { \sum_{d \in \Delta}  C^I(d) } =   \frac{  C^I(x) } { M(C^I) } 
\end{align*}

Observe that the likelihood $P(x | C^I)$ decreases when the size of $C^I$ increases, an effect that recalls about the {\em size principle} by  Tenenbaum and Griffiths \cite{tengrif2001}); this principle is at the basis of their explanation of the numerosity and variability effects in category generalization.

Gliozzi and Plunkett \cite{CogSci2017} 
have shown that in self-organising  maps the  numerosity and the variability of the known instances of a category affect the quality of a category representation: the numerosity of known examples of a category improves the precision of the category representation whereas the variability of these examples diminishes this precision. 
The probabilistic interpretation of self-organising maps given above is in agreement with their experimental results.

For variability, if  the variability increases, $d_{max,{C_i}}= max_{x \in C} \| x- BMU_x\|$ increases as well. 
Then, $rd(y,C_i)$ decreases and consequently $C_i^I(y)=gd_{C_i}(y)$ increases.
As observed by  Gliozzi and Plunkett  \cite{CogSci2017},  an increase in variability of the known category examples generates an increase in generalization;
the higher is variability, the higher the probability of generalization outside the range.

For numerosity, let us consider an increase in the number of input stimuli of category $C_i$, all other things being equal (including the range of values of input stimuli).
Under these conditions, a larger number of input stimuli allows to form a more precise representation of the category (i.e., a representation with a lower  $d_{max,{C_i}}$). In fact, repeatedly updating the weights of the best matching units of ${C_i}$ (and their neighbours), makes their distance from the input stimuli smaller.

While Gliozzi and Plunkett \cite{CogSci2017} have demonstrated that the Numerosity and Variability effects can be accommodated within a simple and psychologically plausible similarity based account,  
in this section we have shown that their generalization degree can indeed be considered as a probability measure.

\normalcolor

\section{The training process as a belief change process} \label{sec:revision} 

We have seen that one can give an interpretation of a self-organising map after the learning phase, as a preferential model.
However, the state of the SOM during the learning phase can as well be represented as a multipreference model, precisely in the same way.
During training, the current state of the SOM corresponds to a model representing the beliefs about the input stimuli considered so far (beliefs concerning the category of the stimuli).

The learning process can then be regarded as a model building process and, in a way, as a belief change process. 
Initially we do not know  the category of the stimuli in the domain $\Delta^s$.
In the initial model, call it $\emme^{som}_0$ (over the domain $\Delta^s$) the interpretation of each concept $C_i$ is empty. 
$\emme^{som}_0$ can be regarded as the model of a knowledge base $K_0$ containing a strict inclusion $C_i \sqsubseteq \bot$, for all $C_i$.

Each time a new input stimulus ($x \in C_i$) is considered, the model is revised adding the stimulus $x$ (and its best matching unit $BMU_{x}$) into the proper category ($C_i$).
Not only the category interpretation is revised by the addition of $x$ and $BMU_{x}$ in $C_i^I$ (so that $C_i \sqsubseteq \bot$ does not hold any more), but also the associated preference relation $<_{C_i}$ is revised as 
the addition of $BMU_{x}$ modifies the set of best matching units $BMU_{C_i}$ for category $C_i$, as well as the distance $d(y,C_i)$ of a stimulus $y$ from $C_i$.  That is, a revision/update step may change the set of conditionals which are satisfied by the model.

At the end of the training process, the final state of the SOM is captured by the model $\emme^{som}$ obtained by a sequence of revision steps which, starting  from  $\emme^{som}_0$, gives rise to a sequence of models $\emme^{som}_0$,$\emme^{som}_{i_1}, \ldots$, $\emme^{som}_{i_r}$ (with $\emme^{som}=\emme^{som}_{i_r}$).
At each step  the knowledge base is not represented explicitly, but the model  $\emme^{som}_{i_j}$ of the knowledge base at step $j$ is used to determine the model at step $j+1$ as a result of the revision/update step ($\emme^{som}_{i_{j+1}}=\emme^{som}_{i_j} \star C_{i_j}(x_{i_j}) $).
The knowledge base $K$ (the set of all the strict and defeasible inclusions and assertions satisfied in $\emme^{som}$) can then be regarded as the knowledge base obtained from $K_0$ through a  sequence of revision/update steps. 
For future work, it would be interesting to study the properties of  this notion of change and compare its properties with the properties of the notions of belief revision \cite{gardenfors,belief-revision,KatsunoMendelzon89,Katsuno-Satoh:91} and iterated belief revision \cite{DP97,GiordanoSL2002,Kern-IsbernerAMAI2004,BoothKR2020} studied in the literature.

\section{Conclusions} \label{conclusions}
The concept-wise multipreference semantics has recently been introduced for dealing with typicality in  description logics \cite{TPLP2020}, based on the idea  
that reasoning about exceptions in ontologies requires taking into account preferences with respect to different concepts and integrating them into a single global preference providing a preferential semantics which allows a standard, KLM style, interpretation of defeasible inclusions.

In this paper, we have explored the relationships between the concept-wise multipreference semantics for the description logic $\lc$ and self-organising maps. 
On the one hand, we have seen that self-organising maps can be given a logical semantics in terms of KLM-style preferential interpretations. 
The model can be used to learn or to validate conditional knowledge from the empirical data used for training and generalization, based on model checking. The learning process in the self-organising map can be regarded as an iterated  
belief change process. 
On the other hand, the plausibility of concept-wise multipreference semantics is supported by the fact that self-organising maps 
are considered as psychologically and biologically plausible neural network models. 

The concept-wise multipreference semantics is related to
 the multipreference semantics for $\alc$ developed by Gliozzi \cite{GliozziAIIA2016}, which is based on the idea of refining the rational closure construction considering the preference relations $<_{A_i}$ associated to different aspects, but the concept-wise multipreference semantics follows a different route concerning both the definition of the preference relations associated with concepts, and  the way of combining them in a single preference relation. 
The idea of having different preference relations, associated to different typicality operators, has been studied by Gil \cite{fernandez-gil} to define a multipreference formulation of the description logic $\alctmin$,  
a typicality DL with a minimal model preferential semantics. 
In this proposal, we associate preferences with concepts,  
and we  combine such preferences into a single global one. 
An extension of DLs with multiple preferences has also been developed by Britz and Varzinczak \cite{Britz2018,Britz2019} to define defeasible role quantifiers and defeasible role inclusions, by associating multiple preference relations with roles.  A related semantics with multiple preferences has also been proposed in the first-order logic setting by Delgrande and Rantsaudis \cite{Delgrande2020}. 

Let us observe that alternative notions of preference combinations can be devised beyond the definition of a global  
preference relation $<$. In some cases, one may want to define preferences with respect to specified criteria of preference combination, as in Brewka's framework of basic preference descriptions  \cite{Brewka04}. In this direction, an algebraic framework for preference combination in Multi-Relational Contextual Hierarchies has  been recently developed by Bozzato et al. \cite{BozzatoIclp2021}.
A related problem of commonsense concept combination has been addressed in a probabilistic extension of the typicality description logic $\alctr$ 
 \cite{Lieto2018}.
Another simple approach to concept combination exploits the fuzzy interpretation of concepts, as introduced in Section \ref{sec:fuzzy_SOM}, for associating preferences to complex concepts \cite{JELIA2021}. In fact, a fuzzy interpretation induces a preference ordering on the domain for each concept. This approach has been exploited for developing a  
fuzzy-multipreference semantics for Multilayer Perceptrons, which allows a deep neural network to be regarded  as a weighted conditional knowledge base.
The approach exploits a combination of the fuzzy and the multipreference semantics considered in this paper.
A fuzzy extension of preferential logics has been previously studied by Casini and Straccia \cite{CasiniStraccia13_fuzzyRC} for G\"odel logic, based on the Rational closure construction.

The logical interpretation of a self-organising map through a fuzzy DL interpretation allows properties of the SOM, expressed as fuzzy concept inclusions, to be verified by model checking.
The correspondence between neural network models and fuzzy systems has been first investigated by Kosko in his seminal work  \cite{Kosko92}.
In his view, ``at each instant the n-vector of neuronal outputs defines a fuzzy unit or a fit vector. Each fit value indicates the degree to which the neuron or element belongs to the n-dimentional fuzzy set." In our approach, in a fuzzy interpretation of a SOM, each concept (representing a learned category) is regarded as a fuzzy set over a domain (here, a set of input stimuli) 
which is the usual way of viewing concepts in fuzzy description logics \cite{Straccia05,LukasiewiczStraccia08,BobilloStraccia16}. 
This allows the validation of fuzzy inclusion axioms over the model of the SOM, as usual in fuzzy DLs. In this direction, one could also consider using fuzzy modifiers ({\em very, slightly}, etc.), which have also been introduced in fuzzy DLs \cite{LukasiewiczStraccia08}, for the formulation of properties to be checked over the fuzzy model of the SOM (e.g., are very big animals either elephants or hippos with a degree $\geq 0.8$?).

The problem of learning fuzzy rules has been widely investigated in the context of fuzzy description logics \cite{LisiStraccia2015,StracciaMucci2015} 
based on other machine learning approaches.
The objective of our work is to show that a fuzzy interpretation of self-organising maps  is possible and natural, and that, it can be used for the validation of fuzzy inclusion axioms over the SOM by model checking.  

Much work has been devoted, in recent years, to the combination 
of neural networks and symbolic reasoning \cite{GarcezBG01,GarcezLG2009,GarcezGori2019}, leading to the definition of new computational models \cite{GarcezGori2020,SerafiniG16,Lukasiewicz2020,PhuocEL21}
and to extensions of logic programming languages
with neural predicates \cite{DeepProbLog18,NeurASP2020}.
Among the earliest 
systems combining logical reasoning and neural learning are the Knowledge-Based Artificial Neural Network (KBANN) \cite{KBANN94} and the Connectionist Inductive Learning and Logic Programming (CILP) \cite{CLIP99}
systems. Penalty Logic \cite{Pinkas95}, a non-monotonic reasoning formalism, was proposed as a mechanism to represent weighted formulas in energy-based 
symmetric connectionist networks (SCNs), 
where the search performed by the SCN for a global minimum may be viewed as a search for a model minimizing penalty. 
The relationships between normal logic programs and connectionist network have been investigated by Garcez and Gabbay \cite{CLIP99,GarcezBG01}
and by Hitzler et al. \cite{HitzlerJAL04}.
None on these approaches addresses the problem of developing a logical interpretation of SOMs.

The logical interpretation of self-organising maps in terms of multipreference  and fuzzy interpretations, besides providing a logical interpretation to SOMs, which may be of interest from the side of explainable AI \cite{Adadi18,Guidotti2019,Arrieta2020}, can  potentially be exploited, as described above, as a basis for an integrated use of self-organising maps and defeasible knowledge bases (resp., fuzzy knoweldge bases).
While a neural network, once trained, is able and fast in classifying the new stimuli (that is, it is able to do instance checking), all other reasoning services such as satisfiability, entailment and model-checking are missing. 
These capabilities are needed for dealing with tasks combining empirical and symbolic knowledge, e.g., 
proving whether the network validates some (strict or conditional) properties;   
learning the weights of a conditional KB from empirical data; 
 combining the defeasible inclusions extracted from a neural network with other defeasible or strict inclusions for inference.
We have seen in the paper that some of these tasks (including property validation and the extraction of defeasible inclusions with a weight) can be achieved over the SOM model.

The idea of constructing a semantic interpretation of a neural network based on a fuzzy and/or  a multipreference semantics  has also been  pursued for  Multilayer Perceptrons (MLPs) in \cite{JELIA2021}.
A deep network is considered after the training phase, when the synaptic weights have been learned, 
to show that  it can be associated with a preferential DL interpretation with multiple preferences, as well as with a semantics based on fuzzy DL  interpretations 
and another one combining fuzzy interpretations with multiple preferences. 
The three semantics allow the input-output behavior of the network to be captured by interpretations built over a set of input stimuli through a simple construction, which exploits the activity level of neurons for the stimuli, 
those units whose meaning we want to reason about,  including hidden units 
(see, for instance, the discussion in \cite{JELIA2021} of the well known Hinton's family example \cite{Hinton1986}). 
Logical properties can be verified over such models by model checking.
Due to the diversity of the two neural models (SOMs and MLPs) we expect that this approach might be extended to other neural network models and learning approaches, by exploiting units activations or notions of distance of a stimulus from a category (as done for SOMs).

The relationships between the logics of common sense reasoning and Multilayer Perceptrons are even deeper, as a deep neural network can be regarded as a conditional knowledge base with weighted conditionals. This has been achieved by developing a concept-wise fuzzy multipreference semantics for a DL with weighted defeasible inclusions \cite{JELIA2021}, as a combination of the fuzzy and the multipreference semantics considered in this paper. 

Whether conditional description logics under a multi-preferential and/or fuzzy semantics can be regarded as possible 
candidates for neuro-symbolic integration is a subject for future investigation.
An issue is 
the development of proof methods for such logics.  
An open problem is whether the notion of fuzzy-multipreference entailment  is decidable, for which DLs fragments
and under which choice of fuzzy logic combination functions.
In the two-valued case multipreference entailment is decidable for weighted $\el$ KBs and can be computed based on ASP encodings \cite{ICLP2021}, by exploiting preferential reasoning in {\em asprin} \cite{BrewkaAAAI15}. This is a first step towards the definition of proof methods for multi-valued extensions of our concept-wise 
preferential semantics based on a notion of faithful interpretations \cite{ECSQARU2021}. 
% Other possible extensions concern the definition of multiple typicality operators, based on the combination of selected concepts, and a temporal extension to capture the transient behavior of Multilayer Perceptrons.
%
Another issue is whether the mapping of multilayer networks  
to weighted conditional knowledge bases can be extended to more complex neural network models, such as Graph neural networks \cite{GarcezGori2020},
or whether different logical formalisms and semantics would be needed.

\medskip
{\bf Acknowledgement:}  We thank the anonymous referees for their helpful comments and suggestions that helped to improve the paper. 
This research has been partially supported by INDAM-GNCS Project 2020. 

%\bibliography{biblioMultipreferenzeFuzzyProbNN2}

\begin{thebibliography}{10}

\bibitem{Adadi18}
A.~Adadi and M.~Berrada.
\newblock Peeking inside the black-box: {A} survey on explainable artificial
  intelligence {(XAI)}.
\newblock {\em {IEEE} Access}, 6:52138--52160, 2018.

\bibitem{Arrieta2020}
A.~Barredo Arrieta, N.~D{\'{\i}}az Rodr{\'{\i}}guez, J.~Del Ser, A.~Bennetot,
  S.~Tabik, A.~Barbado, S.~Garc{\'{\i}}a, S.~Gil{-}Lopez, D.~Molina,
  R.~Benjamins, R.~Chatila, and F.~Herrera.
\newblock Explainable artificial intelligence {(XAI):} concepts, taxonomies,
  opportunities and challenges toward responsible {AI}.
\newblock {\em Inf. Fusion}, 58:82--115, 2020.

\bibitem{rifel}
F.~Baader, S.~Brandt, and C.~Lutz.
\newblock {Pushing the} $\mathcal{EL}$ {envelope}.
\newblock In L.P. Kaelbling and A.~Saffiotti, editors, {\em Proc. IJCAI 2005},
  pages 364--369, Edinburgh, Scotland, UK, August 2005.

\bibitem{handbook}
F.~Baader, D.~Calvanese, D.L. McGuinness, D.~Nardi, and P.F. Patel-Schneider.
\newblock {\em {The Description Logic Handbook - Theory, Implementation, and
  Applications, 2nd edition}}.
\newblock Cambridge, 2007.

\bibitem{baader95b}
F.~Baader and B.~Hollunder.
\newblock Priorities on defaults with prerequisites, and their application in
  treating specificity in terminological default logic.
\newblock {\em Journal of Automated Reasoning (JAR)}, 15(1):41--68, 1995.

\bibitem{BenferhatIJCAI93}
S.~Benferhat, D.~Dubois, and H.~Prade.
\newblock Possibilistic logic: From nonmonotonicity to logic programming.
\newblock In {\em Proc. ECSQARU'93, Granada, Spain, November 8-10, 1993,
  Proceedings}, pages 17--24, 1993.

\bibitem{BobilloStraccia16}
F.~Bobillo and U.~Straccia.
\newblock The fuzzy ontology reasoner {fuzzyDL}.
\newblock {\em Knowledge Based Systems}, 95:12--34, 2016.

\bibitem{BobilloStracciaEL18}
F.~Bobillo and U.~Straccia.
\newblock Reasoning within fuzzy {OWL} 2 {EL} revisited.
\newblock {\em Fuzzy Sets and Systems}, 351:1--40, 2018.

\bibitem{bonattiAIJ15}
P.~A. Bonatti, M.~Faella, I.~Petrova, and L.~Sauro.
\newblock A new semantics for overriding in description logics.
\newblock {\em Artificial Intelligence}, 222:1--48, 2015.

\bibitem{BonattiLW06}
P.~A. Bonatti, C.~Lutz, and F.~Wolter.
\newblock Description logics with circumscription.
\newblock In {\em Proc. KR 2006}, pages 400--410. {AAAI} Press, 2006.

\bibitem{PenalosaARTINT15}
S.~Borgwardt, F.~Distel, and R.~Pe{\~{n}}aloza.
\newblock The limits of decidability in fuzzy description logics with general
  concept inclusions.
\newblock {\em Artificial Intelligence}, 218:23--55, 2015.

\bibitem{BozzatoIclp2021}
L.~Bozzato, T.~Eiter, and R.~Kiesel.
\newblock Reasoning on multi-relational contextual hierarchies via answer set
  programming with algebraic measures.
\newblock {\em Theory and Practice of Logic Programming}, 21:593--609, 2021.

\bibitem{Brewka04}
G.~Brewka.
\newblock A rank based description language for qualitative preferences.
\newblock In {\em Proc. ECAI'2004}, pages 303--307, 2004.

\bibitem{BrewkaAAAI15}
G.~Brewka, J.~P. Delgrande, J.~Romero, and T.~Schaub.
\newblock asprin: Customizing answer set preferences without a headache.
\newblock In {\em Proc. AAAI 2015}, pages 1467--1474, 2015.

\bibitem{Britz2019}
A.~Britz and I.~Varzinczak.
\newblock Contextual rational closure for defeasible {ALC} (extended abstract).
\newblock In {\em Proc. DL 2019, Oslo, Norway}, 2019.

\bibitem{BritzCMMSV21}
K.~Britz, G.~Casini, T.~Meyer, K.~Moodley, U.~Sattler, and I.~Varzinczak.
\newblock Principles of {KLM}-style defeasible description logics.
\newblock {\em {ACM} Trans. Computational Logic}, 22(1):1--46, 2021.

\bibitem{sudafricaniKR}
K.~Britz, J.~Heidema, and T.~Meyer.
\newblock Semantic preferential subsumption.
\newblock In G.~Brewka and J.~Lang, editors, {\em Proc. KR 2008}, pages
  476--484, 2008.

\bibitem{Britz2018}
K.~Britz and I~J. Varzinczak.
\newblock Rationality and context in defeasible subsumption.
\newblock In {\em Proc. FoIKS 2018}, pages 114--132, 2018.

\bibitem{CasiniDL2013}
G.~Casini, T.~Meyer, I.~J. Varzinczak, and K.~Moodley.
\newblock {Nonmonotonic Reasoning in Description Logics: Rational Closure for
  the ABox}.
\newblock In {\em Proc. DL 2013}, pages 600--615, 2013.

\bibitem{casinistraccia2010}
G.~Casini and U.~Straccia.
\newblock {Rational Closure for Defeasible Description Logics}.
\newblock In {\em Proc. JELIA 2010}, volume 6341 of {\em LNCS}, pages 77--90,
  Helsinki, Finland, September 2010. Springer.

\bibitem{Casinistraccia2012}
G.~Casini and U.~Straccia.
\newblock {Lexicographic Closure for Defeasible Description Logics}.
\newblock In {\em Proc. of Australasian Ontology Workshop, vol.969}, pages
  28--39, 2012.

\bibitem{CasiniStraccia13_fuzzyRC}
G.~Casini and U.~Straccia.
\newblock Towards rational closure for fuzzy logic: The case of propositional
  g{\"{o}}del logic.
\newblock In {\em Proc. LPAR-19}, volume 8312 of {\em LNCS}, pages 213--227.
  Springer, 2013.

\bibitem{CasiniStracciaM19}
G.~Casini, U.~Straccia, and T.~Meyer.
\newblock A polynomial time subsumption algorithm for nominal safe
  elo{\(\perp\)} under rational closure.
\newblock {\em Information Sciences}, 501:588--620, 2019.

\bibitem{BoothKR2020}
J.~Chandler and R.~Booth.
\newblock Revision by conditionals: From hook to arrow.
\newblock In {\em Proc. KR 2020}. {AAAI} Press, 2020.

\bibitem{Cintula2011}
P.~Cintula, P.~H\'ajek, and C.~Noguera, editors.
\newblock {\em Handbook of Mathematical Fuzzy Logic}, volume 37-38.
\newblock College Publications, 2011.

\bibitem{DP97}
A.~Darwiche and J.~Pearl.
\newblock On the logic of iterated belief revision.
\newblock {\em Artificial Intelligence}, 89:1--29, 1997.

\bibitem{GarcezBG01}
A.~S. d'Avila Garcez, K.~Broda, and D.~M. Gabbay.
\newblock Symbolic knowledge extraction from trained neural networks: {A} sound
  approach.
\newblock {\em Artif. Intell.}, 125(1-2):155--207, 2001.

\bibitem{GarcezGori2019}
A.~S. d'Avila Garcez, M.~Gori, L.~C. Lamb, L.~Serafini, M.~Spranger, and S.~N.
  Tran.
\newblock Neural-symbolic computing: An effective methodology for principled
  integration of machine learning and reasoning.
\newblock {\em {FLAP}}, 6(4):611--632, 2019.

\bibitem{GarcezLG2009}
A.~S. d'Avila Garcez, L.~C. Lamb, and D.~M. Gabbay.
\newblock {\em Neural-Symbolic Cognitive Reasoning}.
\newblock Cognitive Technologies. Springer, 2009.

\bibitem{CLIP99}
A.~S. d'Avila Garcez and G.~Zaverucha.
\newblock The connectionist inductive learning and logic programming system.
\newblock {\em Applied Intelligence}, 11(1):59--77, 1999.

\bibitem{Delgrande:87}
J.~Delgrande.
\newblock A first-order conditional logic for prototypical properties.
\newblock {\em Artificial Intelligence}, 33(1):105--130, 1987.

\bibitem{Delgrande2020}
J.~Delgrande and C.~Rantsoudis.
\newblock A preference-based approach for representing defaults in first-order
  logic.
\newblock In {\em Proc. NMR2020}.

\bibitem{DuboisP1993_fuzzy_probab}
D.~{Dubois} and H.~{Prade}.
\newblock Fuzzy sets and probability: misunderstandings, bridges and gaps.
\newblock In {\em Proc. IEEE Int. Conf. on Fuzzy Systems}, pages 1059--1068
  vol.2, 1993.

\bibitem{gardenfors}
P.~Gardenf\"{o}rs.
\newblock {\em Knowledge in Flux}.
\newblock MIT Press, 1988.

\bibitem{belief-revision}
P.~Gardenfors and H.~Rott.
\newblock Belief revision.
\newblock In {\em Handbook of Logic in Artificial Intelligence and Logic
  Programming, volume 4, ed. by D. M. Gabbay, C. J. Hogger, and J. A.
  Robinson}. Oxford University Press, 1995.

\bibitem{GelfondLeone02}
M.~Gelfond and N.~Leone.
\newblock Logic programming and knowledge representation - the {A-Prolog}
  perspective.
\newblock {\em Artificial Intelligence}, 138(1-2):3--38, 2002.

\bibitem{fernandez-gil}
Oliver~Fernandez Gil.
\newblock {On the Non-Monotonic Description Logic ALC+T\({}_{\mbox{min}}\)}.
\newblock {\em CoRR}, abs/1404.6566, 2014.

\bibitem{ECSQARU2021}
L.~Giordano.
\newblock On the {KLM} properties of a fuzzy {DL with Typicality}.
\newblock In {\em Sixteenth European Conference on Symbolic and Quantitative
  Approaches to Reasoning with Uncertainty ({ECSQARU} 2021), September 21-24,
  2021}, 2021.
\newblock to appear.

\bibitem{nmr2020}
L.~Giordano and D.~Theseider Dupr{\'{e}}.
\newblock A framework for a modular multi-concept lexicographic closure
  semantics.
\newblock In {\em Proc. NMR2020}.

\bibitem{TPLP2020}
L.~Giordano and D.~Theseider Dupr{\'{e}}.
\newblock An {ASP} approach for reasoning in a concept-aware multipreferential
  lightweight {DL}.
\newblock {\em Theory and Practice of Logic programming, TPLP}, 10(5):751--766,
  2020.

\bibitem{JELIA2021}
L.~Giordano and D.~Theseider Dupr{\'{e}}.
\newblock Weighted defeasible knowledge bases and a multipreference semantics
  for a deep neural network model.
\newblock In {\em Proc. JELIA 2021}, volume 12678 of {\em LNCS}, pages
  225--242. Springer, 2021.

\bibitem{AIJ21}
L.~Giordano and V.~Gliozzi.
\newblock A reconstruction of multipreference closure.
\newblock {\em Artificial Intelligence}, 290, 2021.

\bibitem{CILC2020}
L.~Giordano, V.~Gliozzi, and D.~Theseider Dupr{\'{e}}.
\newblock On a plausible concept-wise multipreference semantics and its
  relations with self-organising maps.
\newblock In {\em Proc. CILC 2020}, volume 2710 of {\em {CEUR} Workshop
  Proceedings}, pages 127--140, 2020.

\bibitem{GiordanoSL2002}
L.~Giordano, V.~Gliozzi, and N.~Olivetti.
\newblock {Iterated Belief Revision and Conditional Logic}.
\newblock {\em Studia Logica}, 70:23--47, 2002.

\bibitem{lpar2007}
L.~Giordano, V.~Gliozzi, N.~Olivetti, and G.~L. Pozzato.
\newblock {P}referential {D}escription {L}ogics.
\newblock In {\em Proc. LPAR 2007}, volume 4790 of {\em LNAI}, pages 257--272,
  2007.

\bibitem{FI09}
L.~Giordano, V.~Gliozzi, N.~Olivetti, and G.~L. Pozzato.
\newblock {ALC+T}: a preferential extension of {D}escription {L}ogics.
\newblock {\em Fundamenta Informaticae}, 96:1--32, 2009.

\bibitem{AIJ15}
L.~Giordano, V.~Gliozzi, N.~Olivetti, and G.~L. Pozzato.
\newblock {Semantic characterization of rational closure: From propositional
  logic to description logics}.
\newblock {\em Artificial Intelligence}, 226:1--33, 2015.

\bibitem{dl2013}
L.~Giordano, V.~Gliozzi, N.~Olivetti, and G.L. Pozzato.
\newblock { Minimal Model Semantics and Rational Closure in Description Logics
  }.
\newblock In {\em Proc. DL 2013}.

\bibitem{ICLP2021}
L.~Giordano and D.~{Theseider Dupr{\'{e}}}.
\newblock Weighted conditional ${\el}$ knowledge bases with integer weights: an
  {ASP} approach.
\newblock In {\em {Proc. ICLP 2021 (Technical Communications)},
  arXiv:2109.07914}, 2021.

\bibitem{GliozziAIIA2016}
V.~Gliozzi.
\newblock Reasoning about multiple aspects in rational closure for {DLs}.
\newblock In {\em Proc. AI*IA 2016}, pages 392--405, 2016.

\bibitem{CogSci2017}
V.~Gliozzi and K.~Plunkett.
\newblock Grounding bayesian accounts of numerosity and variability effects in
  a similarity-based framework: the case of self-organising maps.
\newblock {\em Journal of Cognitive Psychology}, 31(5--6), 2019.

\bibitem{Guidotti2019}
R.~Guidotti, A.~Monreale, S.~Ruggieri, F.~Turini, F.~Giannotti, and
  D.~Pedreschi.
\newblock A survey of methods for explaining black box models.
\newblock {\em {ACM} Computing Surveys}, 51(5):93:1--93:42, 2019.

\bibitem{Hahnle99}
R.~H\"ahnle.
\newblock Advanced many-valued logics.
\newblock In Gabbay D.M. and Guenthner F., editors, {\em Handbook of
  Philosophical Logic}, volume~2, pages 297--395. Springer, Dordrecht, 1999.

\bibitem{Hinton1986}
G.~Hinton.
\newblock Learning distributed representation of concepts.
\newblock In {\em Proceedings 8th Annual Conference of the Cognitive Science
  Society. Erlbaum, Hillsdale, NJ}, 1986.

\bibitem{HitzlerJAL04}
P.~Hitzler, S.~H{\"{o}}lldobler, and A.~Karel Seda.
\newblock Logic programs and connectionist networks.
\newblock {\em J. Applied Logic}, 2(3):245--272, 2004.

\bibitem{Lukasiewicz2020}
P.~Hohenecker and T.~Lukasiewicz.
\newblock Ontology reasoning with deep neural networks.
\newblock {\em J. Artificial Intelligence Research}, 68:503--540, 2020.

\bibitem{KatsunoMendelzon89}
H.~Katsuno and A.~O. Mendelzon.
\newblock A unified view of propositional knowledge base updates.
\newblock In N.~S. Sridharan, editor, {\em Proc. IJCAI 1989}, pages 1413--1419.
  Morgan Kaufmann, 1989.

\bibitem{Katsuno-Satoh:91}
H.~Katsuno and K.~Satoh.
\newblock A unified view of consequence relation, belief revision and
  conditional logic.
\newblock In {\em IJCAI'91}, pages 406--412, 1991.

\bibitem{Kern-Isberner01}
G.~Kern{-}Isberner.
\newblock {\em Conditionals in Nonmonotonic Reasoning and Belief Revision -
  Considering Conditionals as Agents}, volume 2087 of {\em LNCS}.
\newblock Springer, 2001.

\bibitem{Kern-IsbernerAMAI2004}
G.~Kern{-}Isberner.
\newblock A thorough axiomatization of a principle of conditional preservation
  in belief revision.
\newblock {\em Annals of Mathematics and Artificial Intelligence},
  40(1-2):127--164, 2004.

\bibitem{Kern-Isberner2014}
G.~Kern{-}Isberner and C.~Eichhorn.
\newblock Structural inference from conditional knowledge bases.
\newblock {\em Stud Logica}, 102(4):751--769, 2014.

\bibitem{kohonen2001}
T.~Kohonen, M.R. Schroeder, and T.S. Huang, editors.
\newblock {\em Self-Organizing Maps, Third Edition}.
\newblock Springer Series in Information Sciences. Springer, 2001.

\bibitem{Kosko90}
B.~Kosko.
\newblock Fuzziness vs. probability.
\newblock {\em Int. J. General Systems}, 17(2-3):211--240, 1990.

\bibitem{Kosko92}
B.~Kosko.
\newblock {\em Neural networks and fuzzy systems: a dynamical systems approach
  to machine intelligence}.
\newblock Prentice Hall, 1992.

\bibitem{KrausLehmannMagidor:90}
S.~Kraus, D.~Lehmann, and M.~Magidor.
\newblock Nonmonotonic reasoning, preferential models and cumulative logics.
\newblock {\em Artificial Intelligence}, 44(1-2):167--207, 1990.

\bibitem{GarcezGori2020}
L.~C. Lamb, A.~S. d'Avila Garcez, M.~Gori, M.~O.~R. Prates, P.~H.~C. Avelar,
  and M.~Y. Vardi.
\newblock Graph neural networks meet neural-symbolic computing: {A} survey and
  perspective.
\newblock In Christian Bessiere, editor, {\em Proc. {IJCAI} 2020}, pages
  4877--4884, 2020.

\bibitem{PhuocEL21}
D.~Le-Phuoc, T.~Eiter, and A.~Le{-}Tuan.
\newblock A scalable reasoning and learning approach for neural-symbolic stream
  fusion.
\newblock In {\em {AAAI} 2021}, pages 4996--5005. {AAAI} Press, 2021.

\bibitem{whatdoes}
D.~Lehmann and M.~Magidor.
\newblock What does a conditional knowledge base entail?
\newblock {\em Artificial Intelligence}, 55(1):1--60, 1992.

\bibitem{Lehmann95}
D.~J. Lehmann.
\newblock Another perspective on default reasoning.
\newblock {\em Annals of Mathematics and Artificial Intelligence},
  15(1):61--82, 1995.

\bibitem{Lewis:73}
D.~Lewis.
\newblock {\em Counterfactuals}.
\newblock Basil Blackwell Ltd, 1973.

\bibitem{Lieto2018}
A~Lieto and G.L. Pozzato.
\newblock A description logic of typicality for conceptual combination.
\newblock In {\em Proc. {ISMIS} 2018}, volume 11177 of {\em LNCS}, pages
  189--199. Springer, 2018.

\bibitem{LisiStraccia2015}
F.~A. Lisi and U.~Straccia.
\newblock Learning in description logics with fuzzy concrete domains.
\newblock {\em Fundamenta Informaticae}, 140(3-4):373--391, 2015.

\bibitem{LukasiewiczStraccia08}
T.~Lukasiewicz and U.~Straccia.
\newblock Managing uncertainty and vagueness in description logics for the
  semantic web.
\newblock {\em J. Web Semantics}, 6(4):291--308, 2008.

\bibitem{LukasiewiczStraccia09}
T.~Lukasiewicz and U.~Straccia.
\newblock Description logic programs under probabilistic uncertainty and fuzzy
  vagueness.
\newblock {\em Int. J. Approximate Reasoning}, 50(6):837--853, 2009.

\bibitem{Makinson88}
David Makinson.
\newblock General theory of cumulative inference.
\newblock In {\em Proc. NMR 1988}, pages 1--18, 1988.

\bibitem{DeepProbLog18}
R.~Manhaeve, S.~Dumancic, A.~Kimmig, T.~Demeester, and L.~De Raedt.
\newblock Deepproblog: Neural probabilistic logic programming.
\newblock In {\em Proc. NeurIPS 2018}, pages 3753--3763, 2018.

\bibitem{miikkulainen2005}
R.~Miikkulainen, J.A. Bednar, Y.~Choe, and J.~Sirosh.
\newblock {\em Computational Maps in the Visual Cortex}.
\newblock Springer, 2005.

\bibitem{Montes2013}
I.~Montes, J.~Hern{\'{a}}ndez, D.~Martinetti, and S.~Montes.
\newblock Characterization of continuous t-norms compatible with {Zadeh's}
  probability of fuzzy events.
\newblock {\em Fuzzy Sets and Systems}, 228:29--43, 2013.

\bibitem{Nute80}
D.~Nute.
\newblock {\em Topics in Conditional Logic}.
\newblock Reidel, 1980.

\bibitem{OWL}
P.F. Patel-Schneider, P.H. Hayes, and I.~Horrocks.
\newblock {OWL} {W}eb {O}ntology {L}anguage; {S}emantics and {A}bstract
  {S}yntax.
\newblock In {\em http: //www.w3.org/TR/owl-semantics/}, 2002.

\bibitem{Pearl:88}
J.~Pearl.
\newblock {\em Probabilistic Reasoning in Intelligent Systems Networks of
  Plausible Inference}.
\newblock Morgan Kaufmann, 1988.

\bibitem{Pearl90}
J.~Pearl.
\newblock System {Z:} {A} natural ordering of defaults with tractable
  applications to nonmonotonic reasoning.
\newblock In {\em Proceedings TARK'90}, pages 121--135, 1990.

\bibitem{Pensel18}
M.~Pensel and A.~Turhan.
\newblock Reasoning in the defeasible description logic ${ EL}_\bot$ -
  computing standard inferences under rational and relevant semantics.
\newblock {\em Int. J. Approximate Reasoning}, 103:28--70, 2018.

\bibitem{Pinkas95}
G.~Pinkas.
\newblock Reasoning, nonmonotonicity and learning in connectionist networks
  that capture propositional knowledge.
\newblock {\em Artificial Intelligence}, 77(2):203--247, 1995.

\bibitem{SerafiniG16}
L.~Serafini and A.~S. d'Avila Garcez.
\newblock Learning and reasoning with logic tensor networks.
\newblock In {\em Proc. AI*IA 2016}, volume 10037 of {\em LNCS}, pages
  334--348. Springer, 2016.

\bibitem{Smets82}
P.~Smets.
\newblock Probability of a fuzzy event: An axiomatic approach.
\newblock {\em Fuzzy Sets and Systems}, 7(2):153--164, 1982.

\bibitem{Stoilos05}
G.~Stoilos, G.~B. Stamou, V.~Tzouvaras, J.~Z. Pan, and I.~Horrocks.
\newblock Fuzzy {OWL:} uncertainty and the semantic web.
\newblock In {\em Proc. OWLED*05 Workshop}, volume 188 of {\em {CEUR} Workshop
  Proceedings}. CEUR-WS.org, 2005.

\bibitem{Straccia05}
U.~Straccia.
\newblock Towards a fuzzy description logic for the semantic web (preliminary
  report).
\newblock In {\em Proc. {ESWC} 2005}, volume 3532 of {\em Lecture Notes in
  Computer Science}, pages 167--181. Springer, 2005.

\bibitem{StracciaMucci2015}
U.~Straccia and M.~Mucci.
\newblock {pFOIL-DL}: learning (fuzzy) {EL} concept descriptions from crisp
  {OWL} data using a probabilistic ensemble estimation.
\newblock In {\em Proc. {ACM} Symposium on Applied Computing 2015}, pages
  345--352. {ACM}, 2015.

\bibitem{tengrif2001}
J.~B. Tenenbaum and T.~L. Griffiths.
\newblock Generalization, similarity, and bayesian inference.
\newblock {\em Behavioral and Brain Sciences}, 24:629--641, 2001.

\bibitem{KBANN94}
G.~G. Towell and J.~W. Shavlik.
\newblock Knowledge-based artificial neural networks.
\newblock {\em Artificial Intelligence}, 70(1-2):119--165, 1994.

\bibitem{Weydert03}
E.~Weydert.
\newblock System {JLZ} - rational default reasoning by minimal ranking
  constructions.
\newblock {\em Journal of Applied Logic}, 1(3-4):273--308, 2003.

\bibitem{NeurASP2020}
Z.~Yang, A.~Ishay, and J.~Lee.
\newblock Neurasp: Embracing neural networks into answer set programming.
\newblock In C.~Bessiere, editor, {\em Proc. {IJCAI} 2020}, pages 1755--1762,
  2020.

\bibitem{Zadeh1968}
L.~Zadeh.
\newblock Probability measures of fuzzy events.
\newblock {\em J. Mathematical Analysis and Applications}, 23:421--427, 1968.

\end{thebibliography}

\end{document}